\documentclass[conference]{IEEEtran}
\usepackage{times}

\usepackage[numbers, sort&compress]{natbib}
\usepackage{multicol}
\usepackage[bookmarks=true]{hyperref}
\usepackage{graphicx}
\usepackage{caption}

\usepackage[normalem]{ulem}

\usepackage{amsmath} 
\usepackage{amssymb}  
\usepackage[table, xcdraw]{xcolor}
\usepackage[ruled,linesnumbered, noend]{algorithm2e}

\SetCommentSty{mycommfont}
\usepackage{soul}
\usepackage{comment}
\usepackage{multirow}
\usepackage{wrapfig}
\usepackage{subcaption}
\usepackage{balance}
\usepackage{soul, color}
\usepackage{wrapfig}

\usepackage{amsthm}

\newtheorem{mydef}{Definition}
\newtheorem{thm}{Theorem}

\renewcommand{\baselinestretch}{0.993}

\begin{document}

\newcommand{\mgs}{\textsc{Mgs}}

\title{
Multi Graph Search \\ for High-Dimensional Robot Motion Planning
  }


\author{\authorblockN{\textbf{Itamar Mishani and Maxim Likhachev}}
\authorblockA{Robotics Institute, School of Computer Science, Carnegie Mellon University \\
\texttt{\{imishani, maxim\}@cs.cmu.edu}}
}

\twocolumn[{%
  \begin{@twocolumnfalse}
    \maketitle  
    
    \begin{center}
      \includegraphics[width=0.376\textwidth]{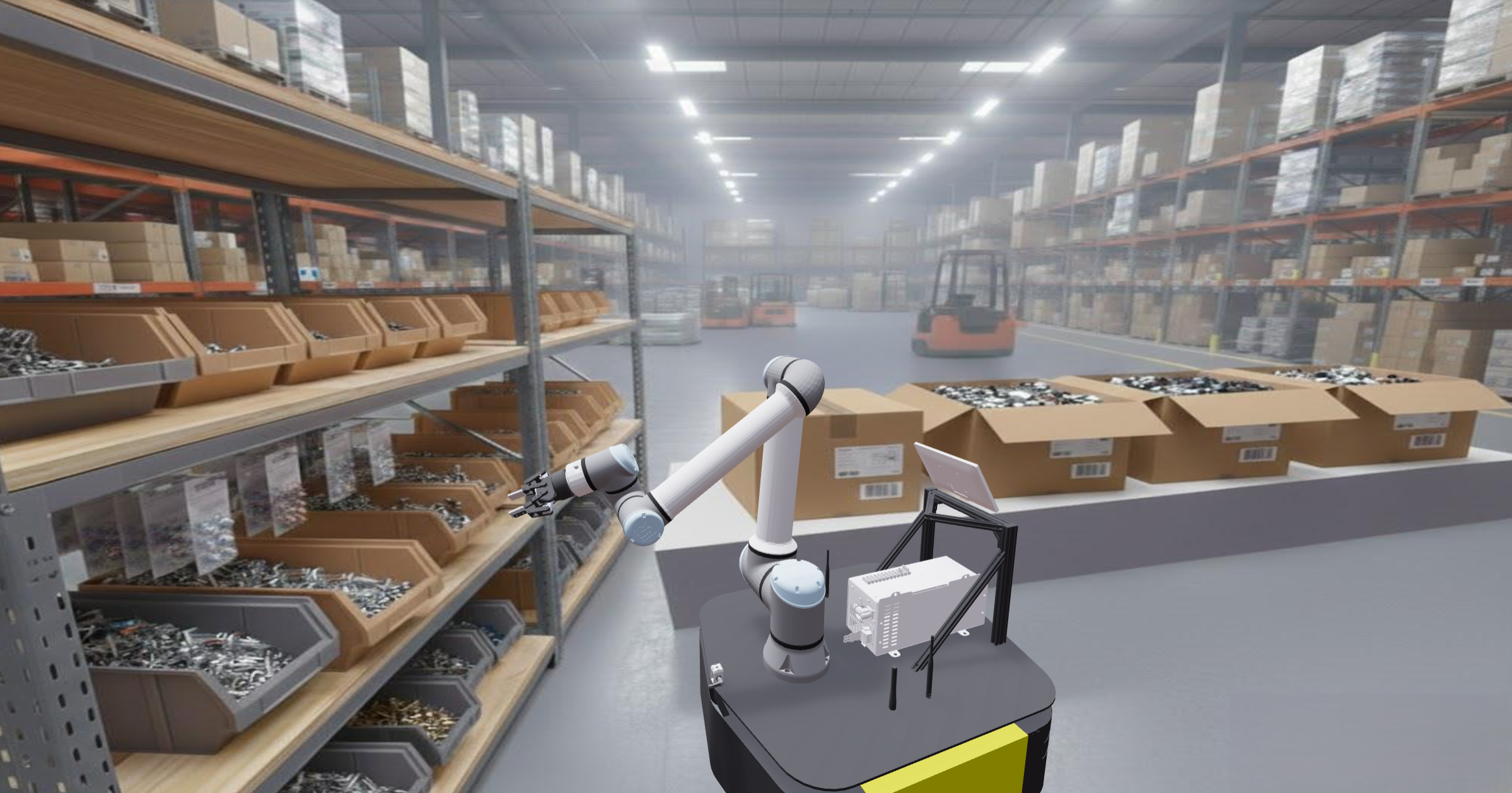}
      \hfill 
      \includegraphics[width=0.614\textwidth]{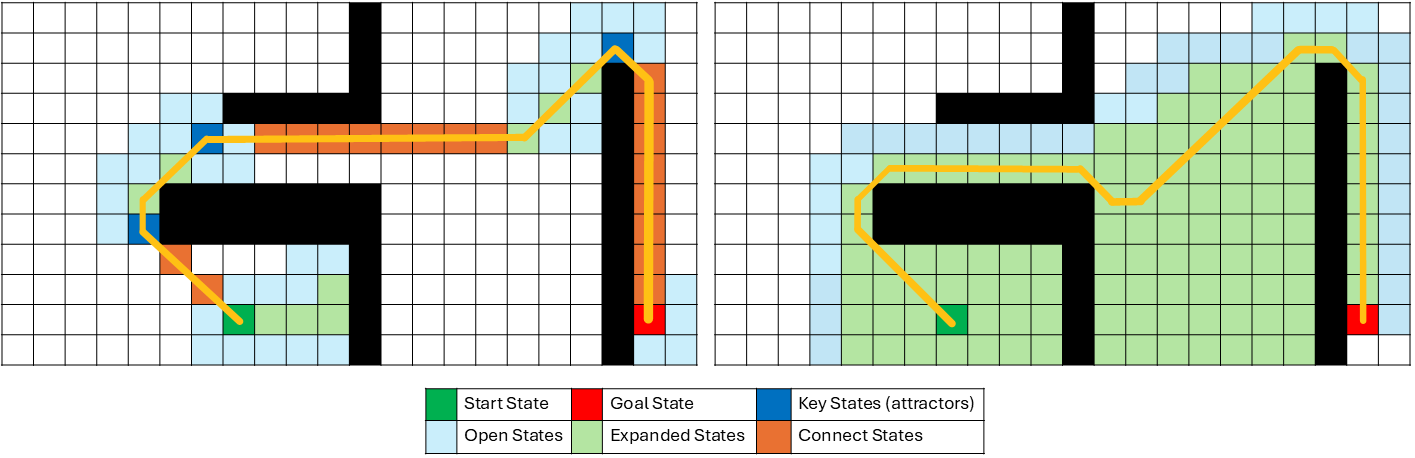}
      
      \captionof{figure}{
        Left: Mobile manipulators in warehouse settings demand efficient, predictable motion planning.
        Middle: \mgs{} anchors search at key states and grows multiple subgraphs simultaneously, yielding a solution with a substantial reduction in search efforts.
        Right: Weighted-A* expands significantly more states to solve the same problem.
        Both searches operate on an 8-connected 2D grid with bounded suboptimality of 10.
        }
      \label{fig:teaser}
    \end{center}
    
    \vspace{1em} 
  \end{@twocolumnfalse}
}]%



\begin{abstract}
Efficient motion planning for high-dimensional robotic systems, such as manipulators and mobile manipulators, is critical for real-time operation and reliable deployment.
Although advances in planning algorithms have enhanced scalability to high-dimensional state spaces, these improvements often come at the cost of generating unpredictable, 
inconsistent motions or requiring excessive computational resources and memory. 
In this work, we introduce \textit{Multi-Graph Search} (\mgs{}), a search-based motion planning algorithm that generalizes classical unidirectional and bidirectional search to a multi-graph setting. \mgs{} maintains and incrementally expands multiple implicit graphs over the state space, focusing exploration on high-potential regions while allowing initially disconnected subgraphs to be merged through feasible transitions as the search progresses. We prove that \mgs{} is complete and bounded-suboptimal, and empirically demonstrate its effectiveness on a range of manipulation and mobile manipulation tasks. Demonstrations, benchmarks and code are available at \href{https://multi-graph-search.github.io/}{\texttt{https://multi-graph-search.github.io/}}.

\end{abstract}

\section{Introduction}
Collision-free motion planning is a fundamental problem in robotics, and a wide range of algorithms have been developed to address it \cite{OMPL2012, srmp}.
Despite this progress, reliably solving motion planning problems for high-dimensional robotic systems---such as mobile manipulators---remains challenging, particularly in high-stakes applications where fast, predictable, and consistent motion planners are essential for deployment (Fig. \ref{fig:teaser}).
A central difficulty lies in the inherent trade-off between planning efficiency and reliability.
Sampling-based methods \cite{rrt,PRM,RRTstar} achieve scalability by exploring the configuration space through randomly sampled collision-free configurations.
While these approaches mitigate the computational burden in high-dimensional spaces by avoiding exhaustive exploration, 
their reliance on randomness often results in unpredictable behavior, inconsistent solutions, and limited guarantees beyond probabilistic guarantees.
Such variability is undesirable in settings that demand repeatability and performance assurances.
In contrast, search-based planning methods systematically expand states from the initial configuration toward the goal using heuristic guidance.
These methods offer deterministic behavior, completeness, high consistency, and bounded suboptimality.
However, their practical utility in high-dimensional systems is often hampered by the curse of dimensionality and a high sensitivity to heuristic accuracy.

This work is motivated by the observation that humans often solve complex tasks by identifying "mental landmarks"~\cite{hayes1979cognitive, generating-options}---promising regions of the state space that guide efficient problem-solving.
We leverage this intuition by introducing a planning algorithm that searches using a set of key states distributed throughout the state space.
Unlike traditional search-based methods that progress unidirectionally from start to goal, our approach simultaneously explores multiple promising directions anchored by these key states, 
which serve as intermediate landmarks to structure and focus the search.
This multi-directional search strategy enables efficient exploration of high-dimensional spaces while preserving the behavior characteristic and theoretical guarantees of classical search-based planners.
In particular, our approach maintains completeness and bounded suboptimality, while substantially improving practical performance in high-dimensional environments.

\noindent\textbf{Our main contributions are}:
\begin{itemize}
    \item A multi-directional heuristic search algorithm that performs simultaneous searches rooted at a set of key states distributed throughout the configuration space.
    \item Theoretical analysis that establishes completeness and bounded suboptimality with respect to discretization guarantees for the proposed algorithm.
    \item A principled strategy for selecting the key states (i.e., graph roots) to identify promising regions for avoiding collisions.
    \item Extensive simulation experiments and real-world demonstrations showing improved performance and reliability over traditional baselines.
\end{itemize}

\section{Related Work}

We review methods for high-dimensional motion planning, focusing on scalability, consistency, and the use of intermediate states to guide search.

\subsection{Classical Planning Approaches and Their Limitations}

\textit{\textbf{Sampling-based methods}} have become the standard for high-dimensional planning over the past two decades precisely because they scale well. 
Pioneering algorithms like Probabilistic Roadmaps (PRM) \cite{PRM} and Rapidly-exploring Random Trees (RRT) \cite{rrt} explore complex configuration spaces by random sampling, mitigating the curse of dimensionality inherent in exhaustive exploration.
However, due to the stochastic nature of sampling-based algorithms, solutions are often inconsistent (similar scenarios are likely to have very different solutions) or appear unintuitive (i.e., highly suboptimal paths).
In high-stakes industrial applications where reliability and repeatability are critical, this unpredictability presents challenges, leading practitioners to often rely on pre-recorded motions or restrictive assumptions~\cite{celemin2022interactive, orthey2023sampling, nardi2016experience, marcucci2025biconvex}.

In contrast, \textit{\textbf{search-based methods}} offer strong theoretical properties---completeness, (bounded sub-) optimality, and consistency---making them appealing for reliable deployment. 
Graph search algorithms like A* \cite{Hart1968} 
systematically expand states according to cost-to-come and heuristic estimates, with rigorous guarantees on solution quality.
Variants of these methods have been successfully applied to robotic navigation \cite{anytime_dynamicA*, D*lite, koenig2005fast, likhachev2009planning} and extended to higher-dimensional systems including manipulation \cite{cohen2010search, cohen2011adaptiveprimitives, egraphs, phillips2015efficient, aine2016mhastar, srmp}, mobile manipulation \cite{aine2016mhastar, chitta2010planning}, and multi-arm systems \cite{shaoulmishani2024xcbs, shaoul2024gencbs}.
However, search-based methods struggle to scale to high dimensions due to the curse of dimensionality, requiring substantial engineering effort in heuristic design and pruning to control computational complexity.

\textit{\textbf{Optimization-based methods}} offer a different perspective, formulating planning as trajectory optimization subject to constraints. Approaches like CHOMP \cite{ratliff2009chomp}, STOMP \cite{kalakrishnan2011stomp}, GPMP2 \cite{gpmp2}, and TrajOpt \cite{trajopt} leverage numerical optimization to generate smooth, dynamically feasible motions.
Yet these methods lack completeness or convergence guarantees and are highly sensitive to initialization, frequently converging to local minima in cluttered environments.

Recently, \textit{\textbf{learning-based methods}} have emerged as a promising direction, using neural networks to learn planning strategies from data \cite{qureshi2019motion, dalal2024neural, mpd, ichter2018learning}. 
Such approaches aim to leverage learned representations to plan, accelerate planning, or improve sampling efficiency.
While these methods show promise in reducing planning time through learned priors, they generally lack the completeness and bounded suboptimality guarantees essential for safety-critical applications, and their generalization to novel environments remains a challenge.

\subsection{Bidirectional and Multi-Directional Search Strategies}

A key insight from both sampling-based and search-based planning is that exploring from multiple directions can dramatically improve efficiency.
In sampling-based planning, bidirectional variants like RRT-Connect \cite{RRT-Connect} accelerate planning by simultaneously growing trees from start and goal, greedily attempting connections between frontiers to efficiently navigate narrow passages and reduce exploration redundancy.
The success of RRT-Connect demonstrates that connecting two exploration frontiers can be far more efficient than unidirectional search.
Extensions to multi-directional sampling-based methods \cite{multidirectionalrrt} have been explored, but rely on random selection of additional root states, inheriting the unpredictability characteristic of sampling-based approaches.
For search-based planning, bidirectional search strategies have been explored where frontiers expand simultaneously from start and goal \cite{pohl1969bi, kaindl1997bidirectional}.
While bidirectional heuristic searches often struggle with frontier alignment (the ``missile metaphor'' \cite{pohl1969bi}), approaches like A*-Connect \cite{islam2016connect} actively guide convergence while preserving suboptimality bounds \cite{lavasani2025anchor}.

The success of bidirectional search naturally motivates exploring from multiple frontiers, yet generalizing beyond two frontiers faces key technical challenges: identifying where to root searches is non-trivial (randomly sampling key states \cite{multidirectionalrrt} reintroduces the unpredictability search-based methods aim to avoid), coordinating multiple frontiers to ensure completeness and bounded suboptimality is algorithmically challenging, and efficiently merging disconnected subgraphs requires careful algorithmic design.
These challenges help explain why multi-directional search has remained largely unexplored in search-based motion planning.

\subsection{Subgoal-Guided and Region-Based Planning}

The intuition of guiding planners through high-potential regions or intermediate subgoals has been explored across multiple research communities.
In the heuristic search literature, landmark-based planning \cite{richter2008landmarks} has been widely employed to enhance heuristic estimates and decompose complex problems into manageable sequences of subgoals \cite{richter2010lama, pommerening2013incremental, segovia2022scaling, landmarksprog}.
However, these classical approaches typically enforce that landmarks be satisfied in a fixed order as part of the final solution, making them distinct from \mgs{}, which treats key states as \textit{optional} seeds for simultaneous multi-directional exploration rather than mandatory waypoints.

Similarly, in reinforcement learning, subgoals have been used to structure hierarchical learning and enable faster exploration \cite{goel2003subgoal, eysenbach2019search, bagaria2021skill}. 
Works in hierarchical RL leverage subgoals to decompose tasks into manageable subtasks, allowing agents to learn reusable skills and accelerate convergence.
This principle---that strategically chosen intermediate objectives can dramatically improve search efficiency---is equally applicable to motion planning.
By identifying and leveraging key states distributed throughout the configuration space, we can focus exploration toward promising regions and substantially accelerate planning, analogous to how subgoals guide learning in RL.

Translating this intuition into practice requires principled methods for identifying key states or promising regions. 
Recently, methods in search-based planning \cite{ctmp, mishani2023constanttime} and optimization-based planning \cite{mp_aroundCS} propose to decompose the environment into promising regions (e.g., convex regions or regions centered around key states) to facilitate more efficient planning and have shown to be especially effective in industrial settings.
Yet these methods require substantial preprocessing and strong assumptions about environment structure, limiting applicability in dynamic settings.

\mgs{} enables multi-directional search-based planning with principled online key state selection, 
combining scalability with theoretical guarantees.

\section{Preliminaries}
\label{sec:preliminaries}

We begin by formally defining the motion planning problem and its graph-based representation.
Next, we introduce focal search, a bounded suboptimal search technique that enables efficient exploration through inadmissible heuristics, forming the basis for our suboptimality guarantees.
Finally, we review attractor-based methods and introduce notations for root selection, providing the foundation for identifying strategic key states through workspace reasoning.

\subsection{Problem Formulation}
\label{subsec:problem_formulation}

We consider collision-free motion planning for a robot $\mathcal{R}$ with configuration space $\mathcal{C} \subseteq \mathbb{R}^d$,
where $d$ denotes the number of degrees of freedom.
The configuration space is partitioned into free space $\mathcal{C}_{\text{free}}$ and obstacle space $\mathcal{C}_{\text{obs}}$, such that $\mathcal{C} = \mathcal{C}_{\text{free}} \cup \mathcal{C}_{\text{obs}}$ and $\mathcal{C}_{\text{free}} \cap \mathcal{C}_{\text{obs}} = \emptyset$.
The robot operates in a world $\mathcal{W} \subseteq \mathbb{R}^3$, with end effector workspace $\mathcal{W}_{\text{eff}} \subseteq \mathcal{W}$ and task space $\mathcal{T} \subseteq SE(3)$.
A motion planning problem instance is defined by a start configuration $q_{\text{start}} \in \mathcal{C}_{\text{free}}$ and a goal condition (predicate) $\phi_{\text{goal}}: \mathcal{C} \rightarrow \{0,1\}$ that specifies the set of satisfying goal configurations $\mathcal{C}_{\text{goal}} = \{q \in \mathcal{C}_{\text{free}} \mid \phi_{\text{goal}}(q) = 1\}$.
In this work, we consider goal conditions defined either in configuration space or in the end effector task space
(e.g., $\phi_{\text{goal}}(q) = 1$ if $FK(q) \in \mathcal{T}_{\text{goal}}$, where $\mathcal{T}_{\text{goal}} \subseteq SE(3)$ is a desired task space region).
The objective is to find a collision-free path $\pi: [0,1] \rightarrow \mathcal{C}_{\text{free}}$ such that $\pi(0) = q_{\text{start}}$ and $\phi_{\text{goal}}(\pi(1)) = 1$, while minimizing a cost function $c(\pi)$.

\textbf{Graph Search Representation.}
To leverage graph search techniques, we discretize the configuration space $\mathcal{C}$ into a graph $G = (V, E)$, where $V$ is the set of vertices representing discrete configurations and $E$ is the set of edges representing feasible transitions between configurations.
Each vertex $v \in V$ corresponds to a configuration $q_v \in \mathcal{C}_{\text{free}}$.
Edges are defined by applying motion primitives $\mathcal{M} = \{m_1, \ldots, m_k\}$---pre-defined local motions---and an edge $(u, v) \in E$ is added only if the motion primitive connecting $q_u$ to $q_v$ is collision-free, as verified through collision checking.
The cost of traversing an edge $(u, v)$ is denoted by $c(u, v)$, typically defined as uniform cost or as the length of the motion primitive.
Since it is not feasible to enumerate all configurations in high-dimensional spaces, we construct an implicit graph where vertices and edges are generated on-the-fly during the search process as it expands already encountered states and applies motion primitives to generate new states.

\subsection{Focal Search for Sub-graph Exploration}
\label{subsec:focal_search}

Focal search \cite{anytime_focal} is a bounded suboptimal search technique that enables efficient exploration by relaxing strict optimality requirements.
In standard A* search, states are expanded in order of their $f$-value, where $f(s) = g(s) + h(s)$, with $g(s)$ being the cost-to-come and $h(s)$ an admissible heuristic estimate of the cost-to-go.
While this guarantees optimality, it can be inefficient when the heuristic provides weak guidance in complex environments.

Focal search introduces a bounded relaxation by maintaining two priority queues: \texttt{OPEN} (ordered by $f$-value) and \texttt{FOCAL} (a subset of \texttt{OPEN}).
Given a suboptimality bound $\epsilon \geq 1$, \texttt{FOCAL} contains all states $s \in$ \texttt{OPEN} with $f(s) \leq \epsilon \cdot f_{\min}$, where $f_{\min}$ is the minimum $f$-value in \texttt{OPEN}.
During expansion, focal search selects states from \texttt{FOCAL} according to an alternative criterion, often an inadmissible heuristic $\hat{h}(s)$ that provides stronger guidance than the admissible heuristic $h(s)$.
In our multi-graph search framework, focal search is used to provide bounded suboptimality guarantees while exploring individual sub-graphs rooted at key states (roots).

\subsection{Attractor-Based Methods for Root Selection}
\label{subsec:attractor_methods}

\begin{figure}
    \centering
    \includegraphics[width=0.95\columnwidth]{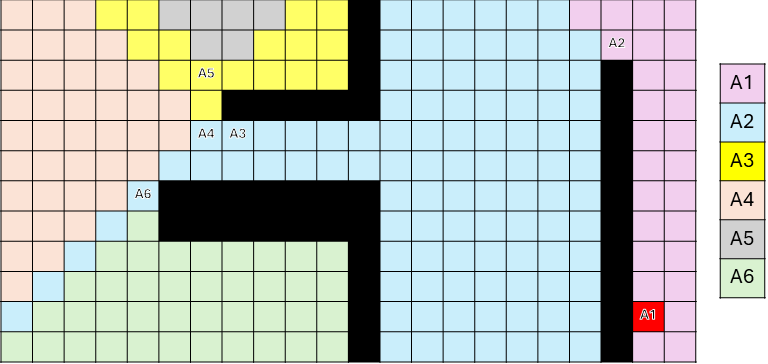}
    \caption{
        Illustration of attractor states ($\text{A}_i$) and their regions of trivial connectivity.
        States within each region can reach the corresponding attractor through greedy tracing. 
        The red cell is the state from which we expand the BFS wavefront. The grid is 8-connected with cardinal cost 1, diagonal cost $\sqrt{2}$, and Euclidean distance as the potential.
        }
    \label{fig:attractor_diagram}
\end{figure}

The concept of \emph{attractor states} was originally introduced for constant-time motion planning~\cite{ctmp} and memory-efficient search~\cite{zou2025attractor}.
In these works, attractors serve as representative states around which neighborhoods are organized, enabling efficient path reconstruction without storing complete paths explicitly.
The key mechanism underlying attractors is \emph{greedy tracing}, which allows states to reach attractors by iteratively following predecessors that minimize a potential function.

Given a search space $\mathcal{S}$ and a non-negative function $h_{\text{potential}}: \mathcal{S} \times \mathcal{S} \rightarrow \mathbb{R}_{\geq 0}$ (e.g., Euclidean distance), greedy tracing is formally defined as follows:

\begin{mydef}[Greedy Predecessor]
For a state $s \in \mathcal{S}$ and target state $s_{\text{target}} \in \mathcal{S}$, a predecessor $s' \in \textsc{Pred}(s)$ is the \emph{greedy predecessor} of $s$ with respect to $s_{\text{target}}$ if
\[
s' = \arg\min_{s'' \in \textsc{Pred}(s)} h_{\text{potential}}(s'', s_{\text{target}})
\]
\end{mydef}

Note that a tie-breaking rule is applied to ensure uniqueness of the greedy predecessor.

\begin{mydef}[Greedy Tracing]
Greedy tracing with respect to $h_{\text{potential}}$ from state $s$ toward target state $s_{\text{target}}$ iteratively selects the greedy predecessor of $s$ with respect to $s_{\text{target}}$, continuing until reaching $s_{\text{target}}$.
\end{mydef}

An attractor state $a \in \mathcal{S}$ is a state whose neighborhood can reach it via greedy tracing, defining a region of \emph{trivial connectivity} (Fig.~\ref{fig:attractor_diagram}).

\section{Multi-Graph Search}
\label{sec:mgs}

In this section, we present \mgs{} (\textbf{M}ulti-\textbf{G}raph \textbf{S}earch) by formalizing the multi-directional problem, introducing the anchor and connect heuristics for exploration and connectivity, describing the algorithm mechanics, and presenting our workspace-aware root selection method.

\subsection{Multi-Directional Problem Formulation}
\label{subsec:multidirectional_formulation}

Traditional search-based planners expand states from a single start configuration $q_{\text{start}}$ toward the goal condition $\phi_{\text{goal}}$.
This unidirectional approach can be inefficient in high-dimensional spaces, as it may explore large portions of the state space that do not contribute to finding a solution.
To address this limitation, \mgs{} employs a multi-directional search strategy that concurrently explores multiple search sub-graphs rooted at a set of key states (roots) $\{r_1, r_2, \ldots, r_m\}$ distributed throughout the configuration space.
Formally, instead of maintaining a single implicit graph $G = (V, E)$, we construct and explore a collection of sub-graphs, $ \mathcal{G} = \{G_1, G_2, \ldots, G_m\}$, where each sub-graph $G_i = (V_i, E_i)$ is an undirected graph rooted at key state $r_i \in \mathcal{C}_{\text{free}}$.
The roots include the start configuration ($r_1 = q_{\text{start}}$), intermediate key states (Section~\ref{subsec:root_selection}), and optionally the goal configuration ($r_m = q_{\text{goal}}$) when specified\footnote{For multi-goal problems, each goal configuration can serve as a root, up to a maximum of $m$ sub-graphs. Here, we assume a single goal configuration or end-effector pose.}.
We distinguish between two types of searches: the \textit{anchor search} $G_1$ rooted at $r_1 = q_{\text{start}}$ uses focal search (Section~\ref{subsec:focal_search}) with suboptimality bound $\epsilon \geq 1$ to explore toward the goal, while \textit{connect searches} $G_2, \ldots, G_m$ use single-queue search (\texttt{OPEN} only) to facilitate connections between regions of the configuration space.

Each sub-graph grows independently---the anchor search $G_1$ expands states from its \texttt{FOCAL} queue, while connect searches $G_2, \ldots, G_m$ expand from their respective \texttt{OPEN} queues.
A solution is found when there exists a connected path through the union of sub-graphs from $q_{\text{start}}$ to a configuration satisfying $\phi_{\text{goal}}$.
Formally, we seek a path $\pi$ that traverses connected sub-graphs:
\[
\pi = \pi_1 \oplus \pi_2 \oplus \cdots \oplus \pi_k
\]
where each $\pi_i$ is a path segment within a sub-graph $G_{j_i} \in \mathcal{G}$, $\pi_1$ starts at $q_{\text{start}}$ in $G_1$, consecutive segments connect at shared vertices, and $\phi_{\text{goal}}(\pi(1)) = 1$.

\subsection{Search Heuristics}
\label{subsec:heuristics}

The efficiency of \mgs{} relies on complementary heuristics that guide the anchor and connect searches:

\noindent\textbf{Anchor Search Heuristics.}
The anchor search $G_1$ uses focal search with two heuristics: \texttt{OPEN} is ordered by $f(s) = g(s) + h(s)$ where $h$ is admissible (e.g., joint Euclidean distance), and \texttt{FOCAL} is ordered by an inadmissible heuristic $\hat{h}(s)$ (task space distance) as described in Section~\ref{subsec:focal_search}.

\noindent\textbf{Connect Search Heuristics.}
Each connect search $G_i$ ($i \geq 2$) uses a single heuristic based on distance to the other sub-graphs:
\[h_{\text{connect}}(s) = \min_{s' \in \textsc{Frontier}(\mathcal{G} \setminus \{G_i\})} d(s, s')\] 
Where the distance function $d(s, s')$ is a pairwise heuristic (e.g., Euclidean distance in configuration space or workspace) measuring proximity between states $s$ and $s'$, and $\textsc{Frontier}(\mathcal{G} \setminus \{G_i\})$ denotes the set of states across all the \texttt{OPEN} lists of the sub-graphs except $G_i$.
In other words, $h_{\text{connect}}(s)$ corresponds to the Front to Front (F2F) heuristic \cite{de1983bidirectional}, guiding connect searches to prioritize expanding states closest to the frontiers of other sub-graphs.

\subsection{Multi-Graph Search Algorithm}
\label{subsec:mgs_algorithm}

Algorithm~\ref{alg:mgs} presents \mgs{}, our multi-directional graph search algorithm.
\mgs{} initializes by selecting root configurations (Lines~\ref{line:mgs}--\ref{line:initsubgraphs}, see Section~\ref{subsec:root_selection}) and constructing an anchor search $G_1$ rooted at $q_{\text{start}}$ with \texttt{OPEN} and \texttt{FOCAL} queues, and connect searches $G_2, \ldots, G_m$ each with a single \texttt{OPEN} queue.

The main loop (Line~\ref{line:mainloop}) alternates between two phases.
In Phase 1 (Anchor Expansion), the anchor search expands the best state from its \texttt{FOCAL} queue according to $\hat{h}$ (Line~\ref{line:anchor_expand}).
If this state satisfies the goal condition, a solution path is returned (Lines~\ref{line:goalcheck}--\ref{line:returnsol}).
Otherwise, the algorithm attempts to connect to nearby frontier states of other sub-graphs via collision-free edges (Line~\ref{line:anchor_connect}).
Successful connections trigger merging into $G_1$ (Line~\ref{line:anchor_merge}).
The merging procedure (\texttt{MergeSubGraphs}) re-roots the merged sub-graph at the connection state, integrates all its vertices and edges into the anchor graph, and propagates updated g-values and updates h-values accordingly, adding newly merged states to the anchor's \texttt{OPEN} and \texttt{FOCAL} queues.
States that were closed in the merged sub-graph are added to the anchor's \texttt{OPEN} queue, allowing re-expansion in the anchor search.
Additionally, if the expanded state was already closed in another sub-graph, the two are merged immediately without requiring an explicit edge connection (Line~\ref{line:closed_check}).

\begin{algorithm}[t]
    \SetAlgoVlined
    \SetCustomAlgoRuledWidth{0.7\textwidth}
    \SetKwFunction{GetRoots}{GetRoots}
    \SetKwFunction{InitializeSubGraphs}{InitializeSubGraphs}
    \SetKwFunction{InitializeFocalSearch}{InitializeFocalSearch}
    \SetKwFunction{InitializeSearch}{InitializeSearch}
    \SetKwFunction{OPEN}{OPEN}
    \SetKwFunction{CLOSED}{CLOSED}
    \SetKwFunction{FOCAL}{FOCAL}
    \SetKwFunction{AddVertex}{AddVertex}
    \SetKwFunction{ExpandState}{ExpandState}
    \SetKwFunction{FindGraphClosingState}{FindGraphClosingState}
    \SetKwFunction{TryToConnect}{TryToConnect}
    \SetKwFunction{MergeSubGraphs}{MergeSubGraphs}
    \SetKwFunction{ReconstructPath}{ReconstructPath}
    \SetKwFunction{GraphsConnectedTo}{GraphsConnectedTo}
    \SetKwFunction{AddVerticesAndEdges}{AddVerticesAndEdges}
    \SetKwFunction{GetConnectingPath}{GetConnectingPath}
    \SetKwFunction{ChooseMergingOrder}{ChooseMergingOrder}
    \SetKwFunction{IsEmpty}{IsEmpty}
    \SetKwFunction{TimeOut}{TimeOut}

    \SetKwInput{KwIn}{\textsc{Input}}
    \SetKwInput{KwOut}{\textsc{Output}}

    \scriptsize
    \caption{Multi-Graph Search (\mgs{})}
    \label{alg:mgs}
    \KwIn{Start configuration $q_{\text{start}} \in \mathcal{C}_{free}$ \newline
          Goal termination condition function $\Phi_{goal}: \mathcal{C} \rightarrow \{0,1\}$ \newline
          Maximum number of sub-graphs $m$ \newline
          Suboptimality bound $\epsilon \geq 1$ \newline
          Admissible heuristic $h: \mathcal{C} \rightarrow \mathbb{R}_{\geq 0}$ \newline
          Inadmissible focal heuristic $\hat{h}: \mathcal{C} \rightarrow \mathbb{R}_{\geq 0}$ \newline
          Connect Heuristic $h_{\text{connect}}: \mathcal{C} \times \mathcal{C} \rightarrow \mathbb{R}_{\geq 0}$
          }
    \KwOut{Collision-free trajectory from $q_{\text{start}}$ to a goal configuration satisfying $\Phi_{goal}$}
    \vspace{4pt}

    \tcp{Initialize sub-graphs rooted at key states}
    roots = \GetRoots($q_{\text{start}}$, $\Phi_{goal}$, $m-1$)
    \label{line:mgs}

    $\mathcal{G} \leftarrow \InitializeSubGraphs(\text{roots})$
    \label{line:initsubgraphs}

    $G_1$.\InitializeFocalSearch($\epsilon, h, \hat{h}$)

    $G_1$.\AddVertex($q_{\text{start}}, g=0$) \tcp*[r]{\color{orange} \scriptsize Initialize start with zero cost}

    \For{$i \in 2$ \KwTo $|\mathcal{G}|$}{
        $G_i$.\InitializeSearch($h_{\text{connect}}$)

        $G_i$.\AddVertex($r_i, g=0$) \tcp*[r]{\color{orange} \scriptsize Initialize root with zero cost}
    }

    \vspace{2pt}
    \tcp{Main search loop: each iteration expands one anchor state and all connect states}
    \vspace{2pt}
    \While{$\neg G_1$.\FOCAL{}.\IsEmpty{} $\land$ $\neg$ \TimeOut{}}{ \label{line:mainloop}
        \tcp{Phase 1: Anchor expansion}
        $q_a \gets G_1$.\FOCAL{}.pop() \label{line:anchor_expand}

        \If{$\Phi_{goal}(q_a)$}{ \label{line:goalcheck}
            \KwRet $G_1$.\ReconstructPath($q_{\text{start}}$, $q_a$) \label{line:returnsol}
        }

        \tcp{Try to connect to other graphs}
        \If{\TryToConnect($q_a$, $\mathcal{G} \setminus \{G_1\}$)}{ \label{line:anchor_connect}
            \For{$G_{\text{connected}} \in$ \GraphsConnectedTo($q_a$, $\mathcal{G}$)}{
                $\pi \gets $ \GetConnectingPath($G_1$, $G_{\text{connected}}$, $q_a$)

                $G_1$.\AddVerticesAndEdges($\pi$)

                $\mathcal{G} \gets$ \MergeSubGraphs($G_1$, $G_{\text{connected}}$, $\pi[|\pi|]$) \tcp*[r]{\color{orange} \scriptsize Anchor always receives merge} \label{line:anchor_merge}
            }
        }

        \tcp{Check if already expanded in another graph, else expand}
        \If{$q_a \in \bigcup_{j=2}^{|\mathcal{G}|} G_j.\CLOSED()$}{ \label{line:closed_check}
            $G_{\text{closed}} \gets$ \FindGraphClosingState($q_a$, $\mathcal{G}$)

            $\mathcal{G} \gets$ \MergeSubGraphs($G_1$, $G_{\text{closed}}$, $q_a$)
        }
        \Else{
            $G_1$.\ExpandState($q_a$)
        }

        \vspace{2pt}
        \tcp{Phase 2: Connect expansions}
        \For{$i \in 2$ \KwTo $|\mathcal{G}|$}{ \label{line:phase2}
            \If{$\neg G_i$.\OPEN{}.\IsEmpty{}}{
                $q_c \gets G_i$.\OPEN{}.pop()

                \If{\TryToConnect($q_c$, $\mathcal{G} \setminus \{G_i\}$)}{ \label{line:connect_try}
                    \For{$G_{\text{connected}} \in$ \GraphsConnectedTo($q_c$, $\mathcal{G}$)}{
                        $G_{\text{from}}, G_{\text{to}} \gets$ \ChooseMergingOrder($G_i$, $G_{\text{connected}}$, $\mathcal{G}$) \label{line:mergeorder}

                        $\pi \gets $ \GetConnectingPath($G_{\text{from}}$, $G_{\text{to}}$, $q_c$)

                        $G_{\text{from}}$.\AddVerticesAndEdges($\pi$)

                        $\mathcal{G} \gets$ \MergeSubGraphs($G_{\text{from}}$, $G_{\text{to}}$, $\pi[|\pi|]$) \label{line:connect_merge}
                    }
                }

                \If{$q_c \in \bigcup_{j=1, j \neq i}^{|\mathcal{G}|} G_j.\CLOSED()$}{
                    $G_{\text{closed}} \gets$ \FindGraphClosingState($q_c$, $\mathcal{G}$)

                    $\mathcal{G} \gets$ \MergeSubGraphs($G_i$, $G_{\text{closed}}$, $q_c$)
                }
                \Else{
                    $G_i$.\ExpandState($q_c$)
                }
            }
        }
    }
    \KwRet \texttt{Failure}

\end{algorithm}

In Phase 2 (Connect Expansions, Line~\ref{line:phase2}), each connect search $G_i$ ($i \geq 2$) expands its best state according to $h_{\text{connect}}$, which prioritizes states near frontier states of other sub-graphs.
Connection attempts and potential merges follow similar logic to Phase 1 (Lines~\ref{line:connect_try}--\ref{line:connect_merge}), with merging order determined by \texttt{ChooseMergingOrder} (Line~\ref{line:mergeorder}) based on which sub-graph is closer to the anchor search.
When two connect searches merge (neither being the anchor), their \texttt{OPEN} queues are combined but closed states are not reopened---the merged sub-graph inherits the closed sets of both, avoiding redundant expansions until the eventual merge into the anchor.
The algorithm terminates when either a solution is found via the anchor search (Line~\ref{line:returnsol}) or when the anchor's \texttt{FOCAL} queue becomes empty (Line~\ref{line:mainloop}), indicating no $\epsilon$-suboptimal solution exists.
Fig.~\ref{fig:teaser} illustrates this on a 2D grid: \mgs{} yields a solution with significantly fewer expansions than weighted A* under the same suboptimality bound.

\subsection{Workspace-Aware Root Selection via Backward BFS}
\label{subsec:root_selection}

\begin{algorithm}[!b]
    \SetAlgoVlined
    \SetCustomAlgoRuledWidth{0.7\textwidth}
    \SetKwFunction{GetRoots}{\textbf{GetRoots}}
    \SetKwFunction{OPEN}{OPEN}
    \SetKwFunction{CLOSED}{CLOSED}
    \SetKwFunction{ATTRACTORS}{ATTRACTORS}
    \SetKwFunction{INSERT}{Insert}
    \SetKwFunction{ExpandState}{ExpandState}
    \SetKwFunction{IsEmpty}{IsEmpty}
    \SetKwFunction{TimeOut}{TimeOut}
    \SetKwFunction{GetRobotConfiguration}{GetRobotConfiguration}
    \SetKwFunction{GetRobotEEgoalPosition}{GetRobotEEgoalPosition}
    \SetKwFunction{Cluster}{Cluster}
    \SetKwFunction{ForwardAttractors}{ForwardAttractors}

    \SetKwInput{KwIn}{\textsc{Input}}
    \SetKwInput{KwOut}{\textsc{Output}}

    \scriptsize
    \caption{Workspace-Aware Root Selection via Backward BFS}
    \label{alg:root_selection}
    \KwIn{Start configuration $q_{\text{start}} \in \mathcal{C}_{free}$ \newline
          Goal termination condition function $\Phi_{goal}: \mathcal{C} \to \{0,1\}$ \newline
          Maximum number of sub-graphs $m$}

    \KwOut{Set of root configurations $\{r_1, r_2, \ldots, r_m\}$}

    \textsc{\textbf{Note}:} Workspace $\mathcal{W}_{\text{eff}}$ is discretized into a 3D occupancy grid
    \vspace{4pt}

    \SetKwProg{getroots}{\GetRoots}{}{}

    \getroots{($q_{\text{start}}$, $\Phi_{goal}$, $m$)}{

        \If {goal configuration is fully specified}{
            $roots \gets \{\Phi_{goal}.\GetRobotConfiguration()\}$ \tcp*[r]{\color{orange} \scriptsize Include goal as root}
        }
        \Else{
            $roots \gets \emptyset$
        }
        \label{line:root_init}

        
        $w.state \gets \Phi_{goal}.\GetRobotEEgoalPosition()$ \tcp*[r]{\color{orange} \scriptsize Start from goal end effector position}

        $w.g \gets 0$ \tcp*[r]{\color{orange} \scriptsize Zero cost at goal}

        $w.attractor \gets w.state$ \label{line:bfs_init} \tcp*[r]{\color{orange} \scriptsize The BFS start state is the first attractor}

        \CLOSED $\gets \emptyset$

        \OPEN $\gets \{w\}$ \tcp*[r]{\color{orange} \scriptsize Initialize with start state (FIFO)}

        \ATTRACTORS $\gets \emptyset$ 
        
        \While{$\neg$ \OPEN.\IsEmpty{} }{ \label{line:bfs_loop}
            $w \gets$ \OPEN.pop()

            \INSERT $w$ into \CLOSED

            \ForEach{neighbor $w'$ of $w$}{
                \If{$w'$ is not in collision $\land$  $w' \notin$ \CLOSED }{
                    \If{$w'.g > w.g + \text{cost}(w, w')$}{
                        $w'.g \gets w.g + \text{cost}(w, w')$ \tcp*[r]{\color{orange} \scriptsize We use uniform cost}
                        $a \gets w.attractor$ \tcp*[r]{\color{orange} \scriptsize Get attractor from parent}

                        $\text{greedy\_pred} \gets \arg\min_{w'' \in \textsc{Pred}(w')} d(w'', a)$ \label{line:greedy_pred} \tcp*[r]{\color{orange} \scriptsize Greedy predecessor toward attractor}

                        \eIf{$\text{greedy\_pred} \neq w$}{
                            $w'.attractor \gets w.state$ \label{line:new_attractor} \tcp*[r]{\color{orange} \scriptsize New attractor found at this branching point}
                            \If{$w.state \notin $ \ATTRACTORS}{
                                \INSERT $(w.state, \GetRobotConfiguration(w))$ into \ATTRACTORS
                            }
                        }{
                            $w'.attractor \gets a$ \tcp*[r]{\color{orange} \scriptsize Inherit attractor from parent}
                        }
                    }
                    \If{$w' \notin$ \OPEN}{
                        \INSERT $w'$ into \OPEN
                    }
                }
            }
        }

        \tcp{Adding the forward attractors when following the computed policy from start to goal}
        $\texttt{forward\_attractors} \gets$ \ForwardAttractors($q_{\text{start}}$, $\Phi_{goal}.\GetRobotEEgoalPosition()$) \label{line:forward_attractors}

        \ForEach{$(w_{\text{attr}}, q_{\text{attr}}) \in \texttt{forward\_attractors}$}{
            \If{$(w_{\text{attr}}, q_{\text{attr}}) \notin \ATTRACTORS$}{
                $\ATTRACTORS \gets \ATTRACTORS \cup \{(w_{\text{attr}}, q_{\text{attr}})\}$
            }
        }
        
        \tcp{Cluster attractors to select up to maximum of $m$ roots}
        \eIf{$|\ATTRACTORS| + |roots| \leq m$}{ \label{line:cluster_check}
            $roots \gets roots \cup \{q_{\text{attr}} \mid (w_{\text{attr}}, q_{\text{attr}}) \in \ATTRACTORS\}$
        }{
            $clusters \gets$ \Cluster(\ATTRACTORS, $m - |roots|$) \label{line:cluster}

            \ForEach{cluster $c \in clusters$}{
                $q_{\text{rep}} \gets$ representative configuration of cluster $c$

                $roots \gets roots \cup \{q_{\text{rep}}\}$
            }
        }

        \KwRet $roots$
    }

\end{algorithm}

Selecting roots $\{r_1, r_2, \ldots, r_m\}$ is critical to \mgs{}'s efficiency.
We seek a principled method for selecting intermediate states that are likely to lie along useful paths from start to goal.
Identifying such waypoints directly in the high-dimensional configuration space $\mathcal{C}$ is computationally expensive.
Instead, we approximate the collision-free regions in the lower-dimensional end effector workspace $\mathcal{W}_{\text{eff}}$ by discretizing it into a 3D occupancy grid based on obstacle geometry.
We then employ backward breadth-first search (BFS) in this discretized workspace to identify intermediate \textit{attractor} states that mark the boundaries of distinct collision-free regions~\cite{zou2025attractor}.
These workspace attractors are mapped back to configuration space to serve as roots.
Alg.~\ref{alg:root_selection} outlines the procedure.

The algorithm proceeds in three stages: backward 3D BFS to compute a workspace policy and identify attractors, forward rollout to select attractors along the start-to-goal path, and clustering to respect the roots budget.
The BFS is initialized from the goal end effector position in the discretized workspace grid, with zero cost and itself as the initial attractor (Line~\ref{line:bfs_init}).
If the goal configuration is fully specified, it is also included as a root (Line~\ref{line:root_init}).
The core insight is that \textit{attractors} emerge at workspace locations where the shortest path from goal diverges due to obstacles~(Fig.~\ref{fig:attractor_diagram}).
As the BFS wavefront expands backward from the goal (Line~\ref{line:bfs_loop}), each state $w$ maintains a reference to its current attractor---the workspace position it is ``attracted toward'' when following a greedy path to the goal.
When expanding a neighbor $w'$, the algorithm checks whether the greedy predecessor toward the current attractor $a$ matches the actual BFS parent $w$ (Line~\ref{line:greedy_pred}).
If they differ, this indicates a branching point where the obstacle geometry forces paths to diverge, and $w$ becomes a new attractor (Line~\ref{line:new_attractor}).
Otherwise, $w'$ inherits the attractor from its parent.
Each attractor is stored along with a corresponding robot configuration obtained via differential inverse kinematics, seeded with the configuration of the previous attractor.
The backward BFS identifies attractors with respect to the goal, but we also want to identify important attractors with respect to the start.
After the BFS completes, the algorithm traces the computed policy forward from the start position to the goal, collecting the attractors encountered along this path (Line~\ref{line:forward_attractors}).
A further discussion and intuition on forward attractors appears in the appendix.
The forward attractors are added to the attractor set, and if the total number of attractors does not exceed the budget $m$, all are included as roots (Line~\ref{line:cluster_check}).
Otherwise, spatial clustering (e.g., $k$-means in workspace) groups nearby attractors, and a representative configuration from each cluster is selected as a root (Line~\ref{line:cluster}).
This ensures the roots remain well-distributed while respecting the computational budget.
The resulting root configurations $\{r_1, \ldots, r_m\}$ capture strategic waypoints that the robot's end effector must navigate around obstacles, providing \mgs{} with informed starting points for its connect searches.

\subsection{Theoretical Properties}
\label{subsec:theoretical_properties}

We establish that \mgs{} inherits the completeness and bounded suboptimality guarantees of focal search.
The following properties hold with respect to the implicit graph $G = (V, E)$ induced by the discretization described in Section~\ref{sec:preliminaries}.

\begin{thm}[Bounded Suboptimality]
\label{thm:suboptimality}
Let $h: \mathcal{C} \to \mathbb{R}_{\geq 0}$ be an admissible heuristic (i.e., $h(q) \leq c^*(q, q_{\text{goal}})$ for all $q \in \mathcal{C}$).
If \mgs{} returns a solution path $\pi$, then $\text{cost}(\pi) \leq \epsilon \cdot c^*$, where $c^*$ is the optimal solution cost and $\epsilon \geq 1$ is the suboptimality bound.
\end{thm}

\begin{proof}[Proof Sketch]
Solutions are returned exclusively through the anchor search $G_1$ when a state $q_a \in \texttt{FOCAL}$ satisfies $\Phi_{\text{goal}}(q_a) = 1$ (Alg.~\ref{alg:mgs}, Line 9).
The anchor maintains the focal search invariant: all states in \texttt{FOCAL} have $f(s) \leq \epsilon \cdot f_{\min}$, where $f_{\min} = \min_{s \in \texttt{OPEN}} f(s)$.
Since $h$ is admissible, $f_{\min} \leq c^*$ throughout the search.
When connect searches merge into the anchor via \texttt{MergeSubGraphs}, the g-values of merged states are recomputed with respect to $q_{\text{start}}$ by propagating costs through the connecting path.
This preserves the focal invariant: merged states are inserted into \texttt{OPEN} with correct g-values, and only enter \texttt{FOCAL} if their f-values satisfy the bound.
Since the returned solution has $g(q_a) = \text{cost}(\pi)$ and $q_a$ was in \texttt{FOCAL}, we have $\text{cost}(\pi) = f(q_a) \leq \epsilon \cdot f_{\min} \leq \epsilon \cdot c^*$.
\end{proof}

\begin{thm}[Bounded Re-Expansions]
\label{thm:reexpansions}
If $h$ is consistent (i.e., $h(q) \leq c(q, q') + h(q')$ for all edges $(q, q')$), then each state is expanded at most twice during \mgs{} execution.
\end{thm}

\begin{proof}[Proof Sketch]
With a consistent heuristic, focal search never re-expands states within a single sub-graph.
A state $q$ may be expanded once in a connect search $G_i$ before merging, and at most once more in the anchor $G_1$ after merging (since merged states are added to \texttt{OPEN} and may be re-expanded with updated g-values).
Once expanded in the anchor with the correct g-value, consistency ensures $q$ will not be expanded again.
\end{proof}

\begin{thm}[Completeness]
\label{thm:completeness}
If a solution exists, \mgs{} will find one (given sufficient time and memory).
\end{thm}

\begin{proof}[Proof Sketch]
The anchor search $G_1$ performs focal search, which is complete: it systematically expands states from \texttt{FOCAL} and will eventually expand all reachable states if no solution is found earlier.
Connect searches and merging only accelerate the search by adding states to the anchor's \texttt{OPEN} queue---they never remove states or prevent exploration.
If a solution path exists from $q_{\text{start}}$ to a goal state, the anchor will eventually either (a) discover it directly through expansion, or (b) discover it through states added via merging.
In either case, the goal state will eventually enter \texttt{FOCAL} and be returned.
\end{proof}


\begin{figure*}[!ht]
    \centering
    \includegraphics[width=0.135\textwidth]{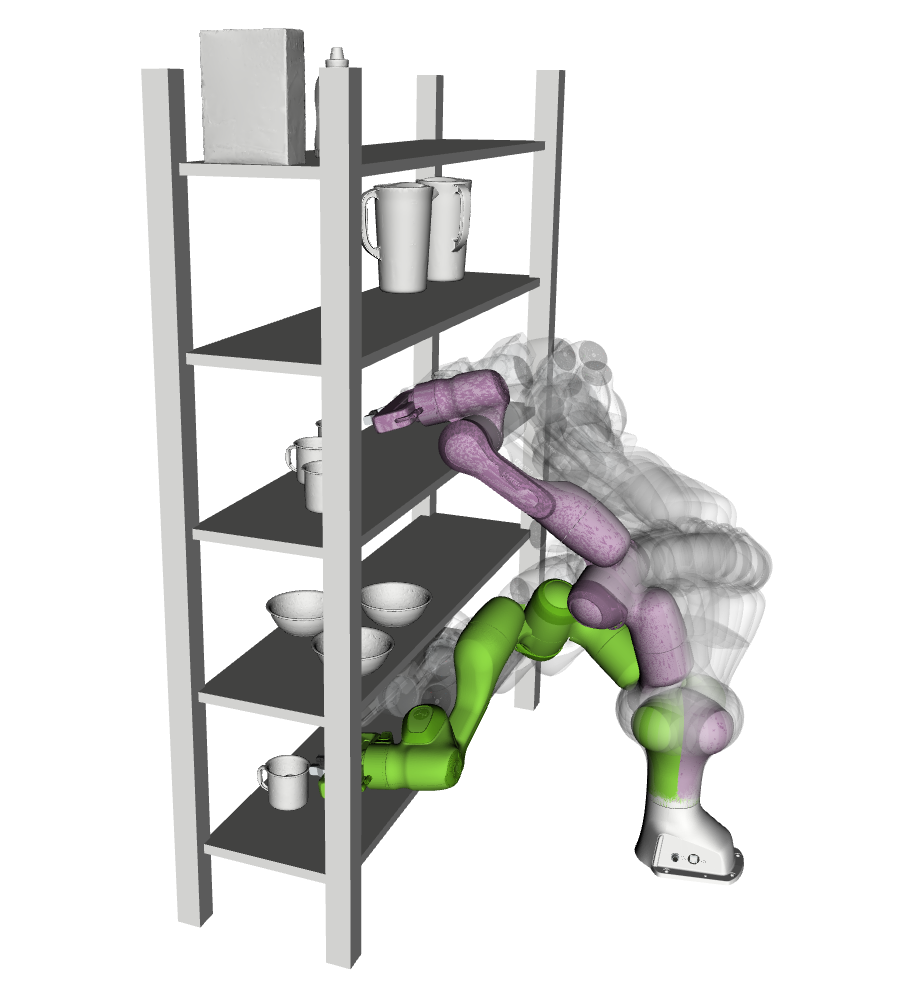}
    \includegraphics[width=0.135\textwidth]{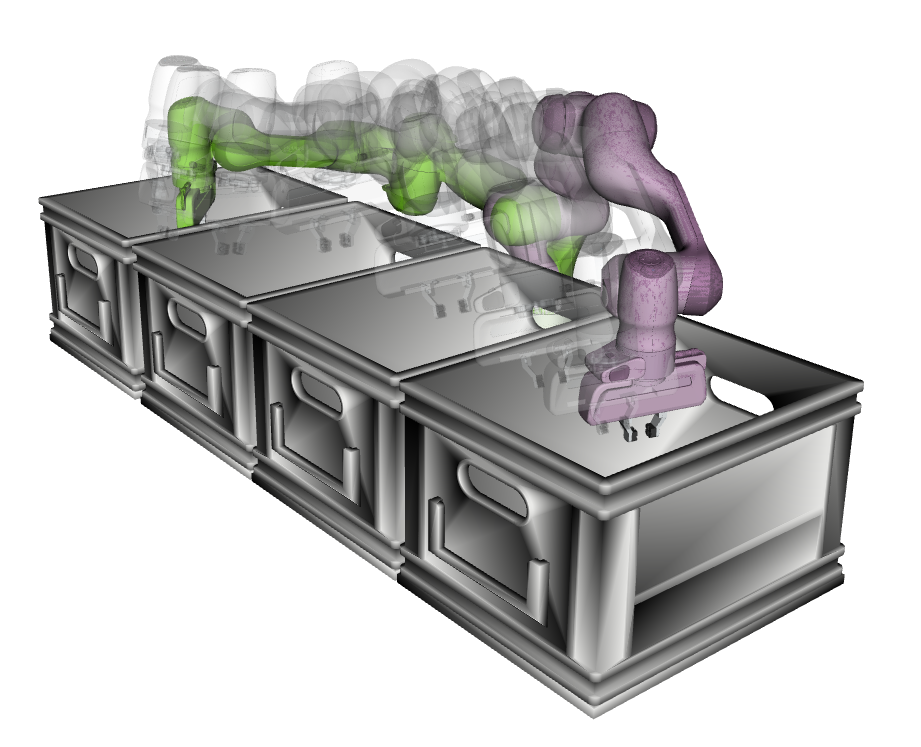}
    \includegraphics[width=0.135\textwidth]{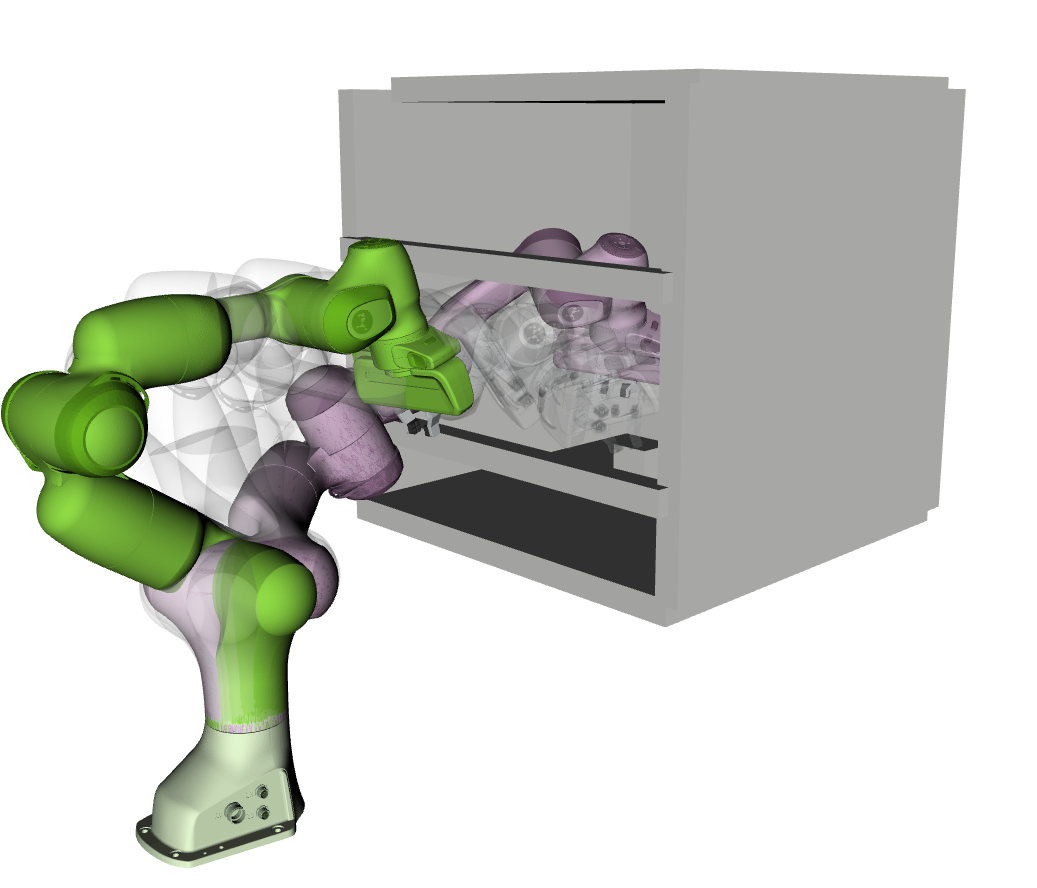}
    \includegraphics[width=0.135\textwidth]{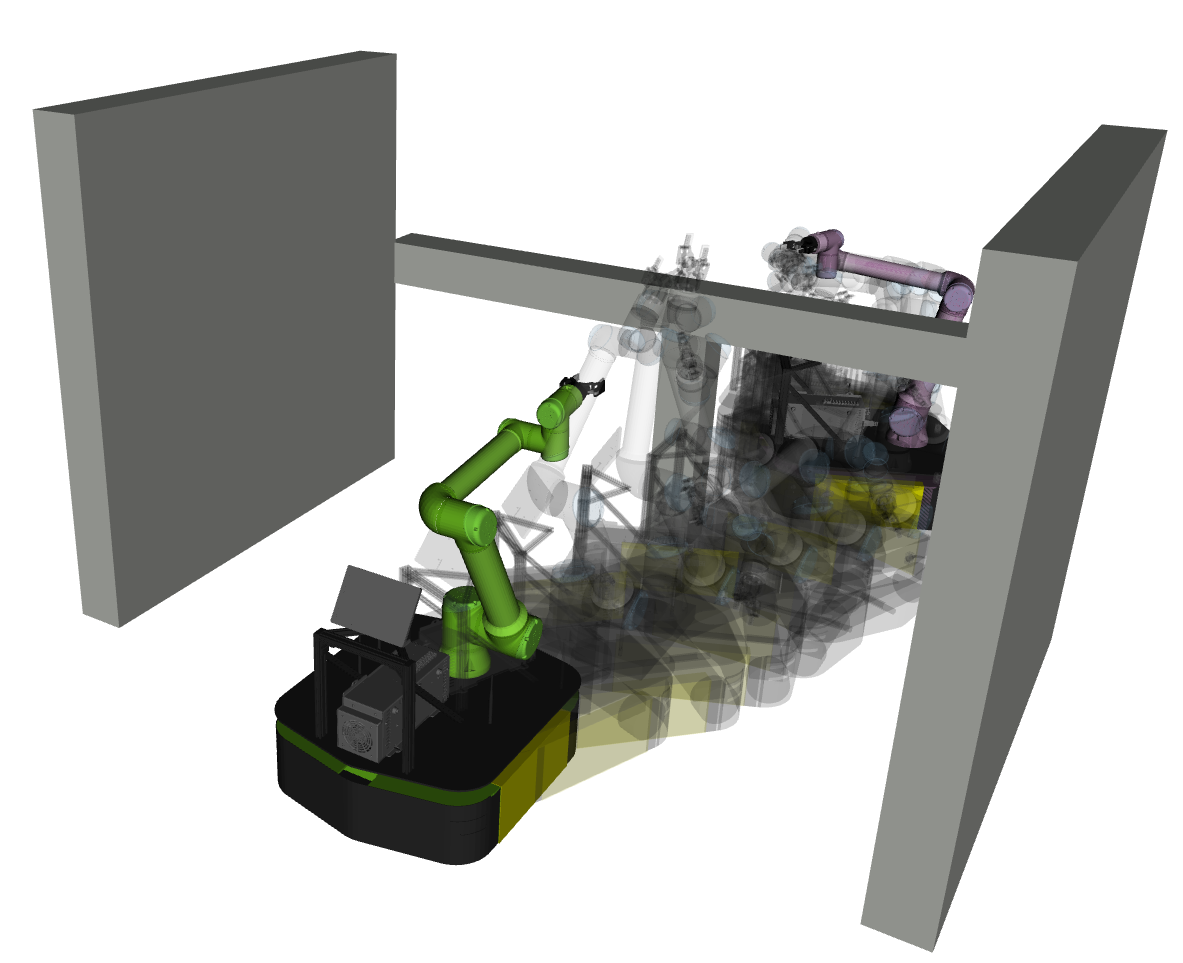}
    \includegraphics[width=0.135\textwidth]{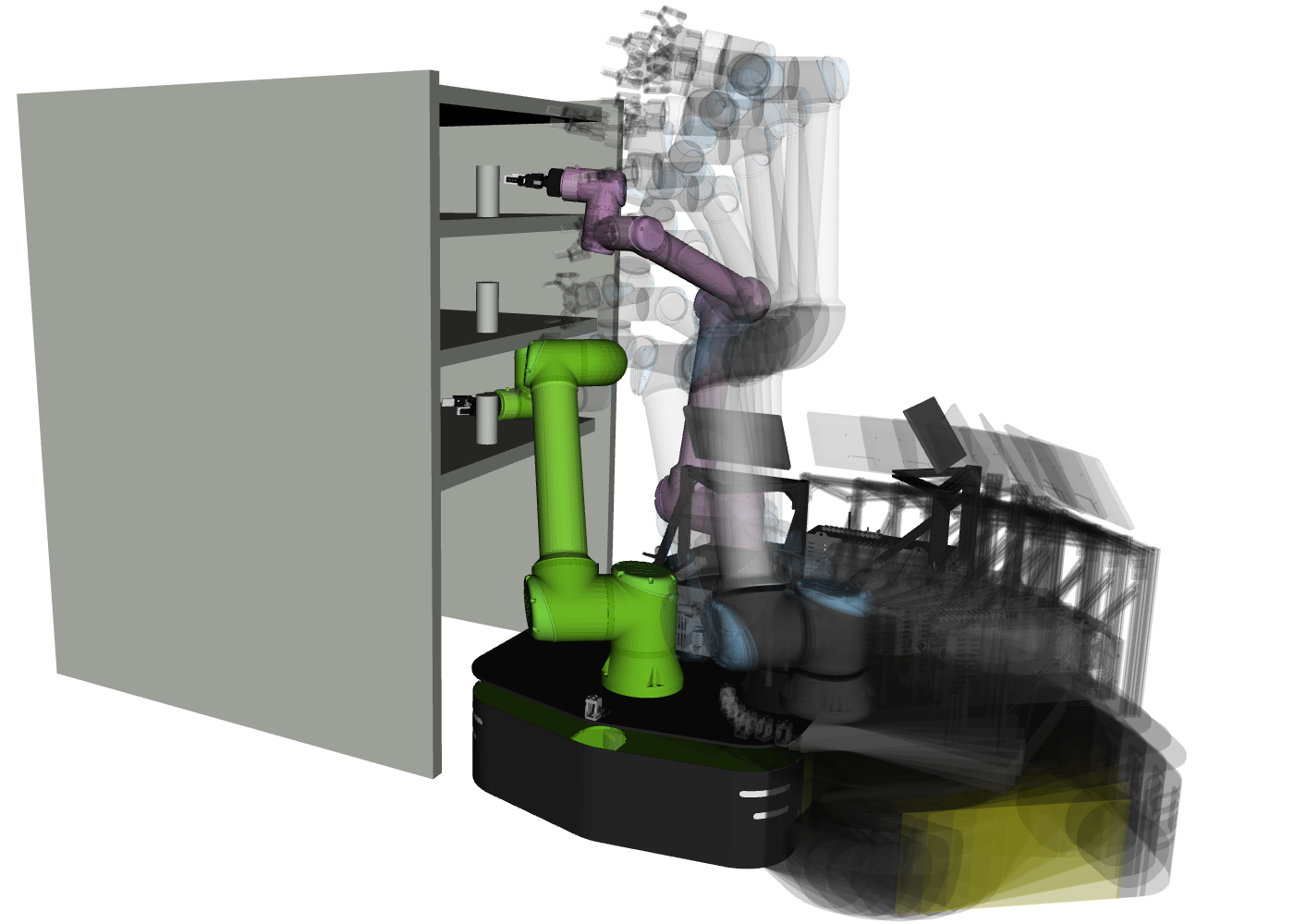}
    \includegraphics[width=0.135\textwidth]{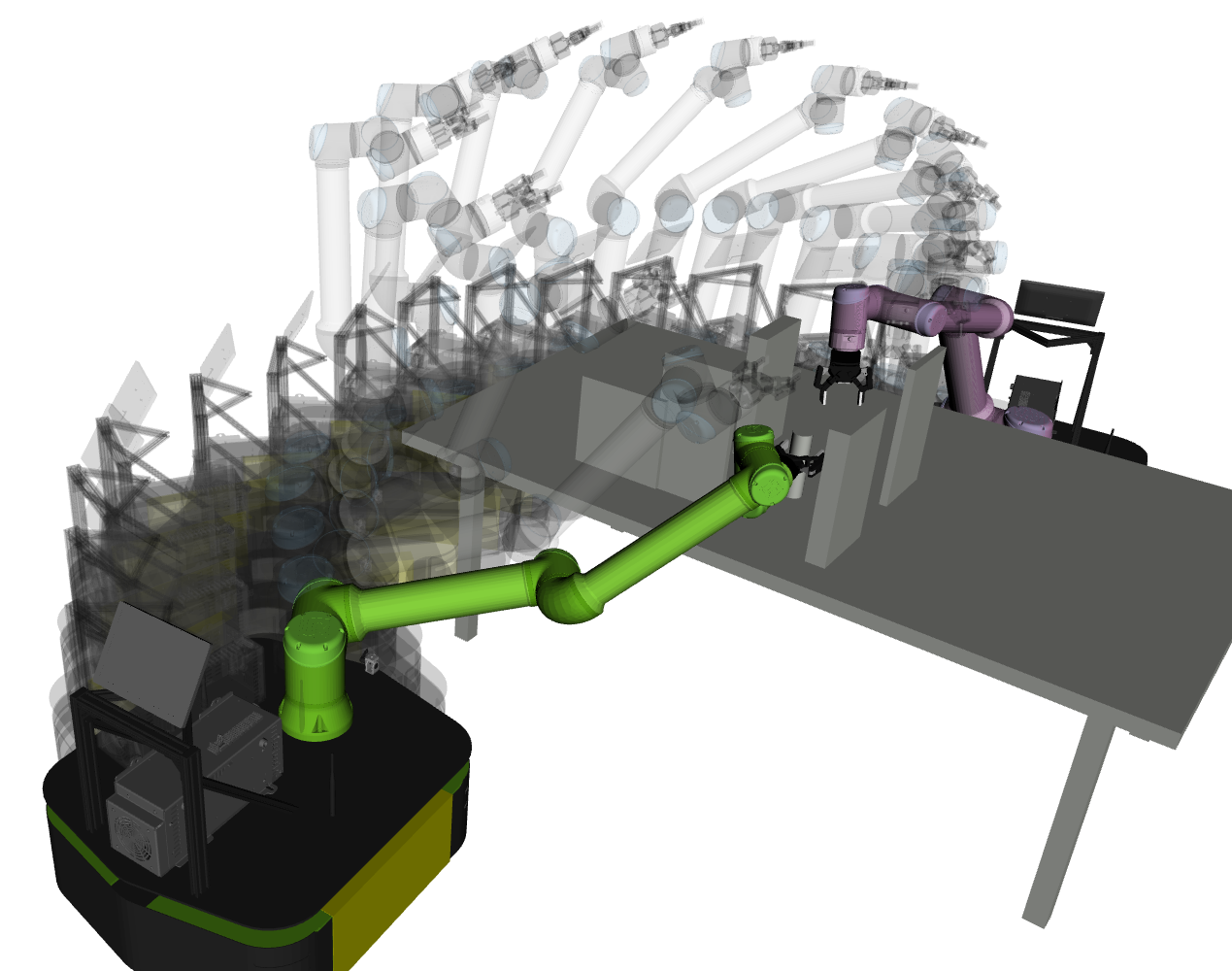}
    \includegraphics[width=0.135\textwidth]{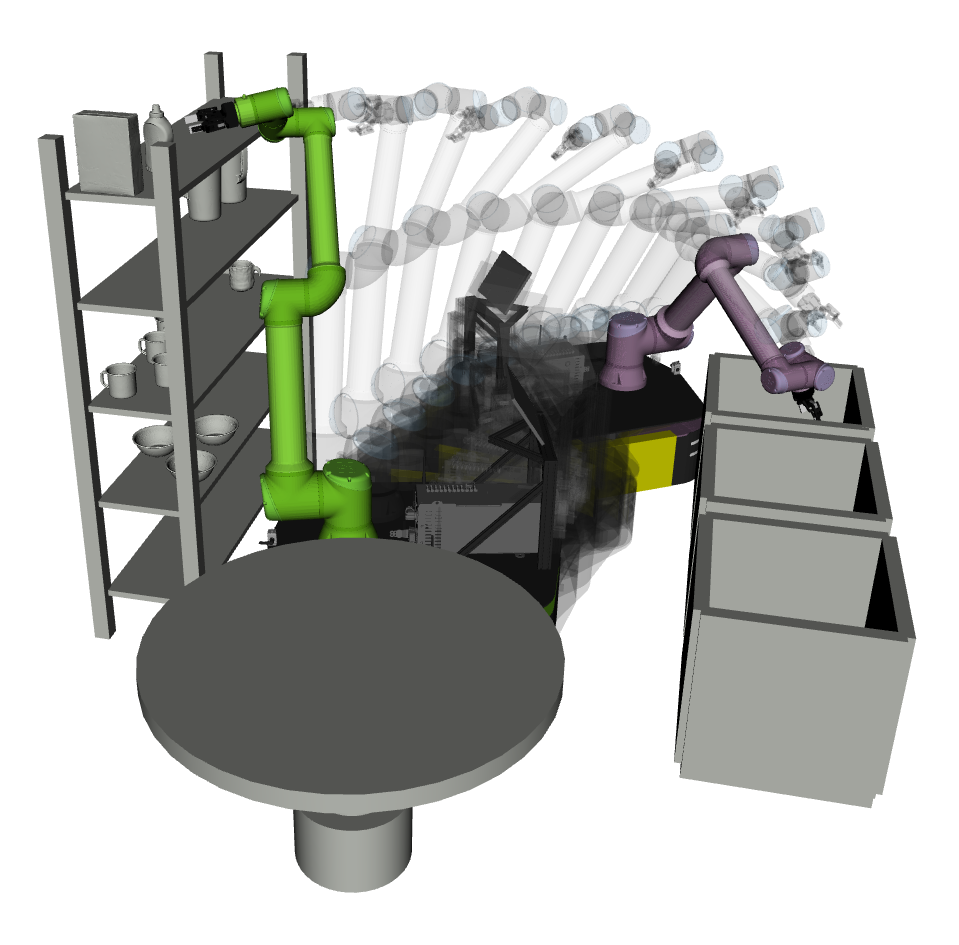}
    \caption{Experimental environments for manipulation (left three: shelf pick-and-place, bin picking, cage extraction) and mobile manipulation (right four: low-clearance passage, deep shelf reach, cluttered table, combined warehouse).}
    \label{fig:environments}
\end{figure*}

\section{Experiments}
\label{sec:experiments}

Our experimental design focuses on two core objectives: demonstrating that \mgs{} can efficiently generate high-quality solutions for complex motion planning tasks, and validating the algorithm's consistency across repeated and perturbed trials.
We compare \mgs{} against state-of-the-art sampling-based baselines provided by OMPL~\cite{OMPL2012} and optimization-based baselines, 
as well as search-based planners via the SRMP framework~\cite{srmp}. 

\subsection{Experimental Setup}
\label{subsec:exp_setup}

We used scenarios from the Motion Benchmarker~\cite{chamzas2021motionbenchmaker} and extended it with additional environments representing realistic industrial challenges such as narrow passages, cluttered workspaces, and confined reaching tasks.
The evaluation covers seven environments spanning manipulation and mobile manipulation (Fig.~\ref{fig:environments}).
We validate performance on two hardware configurations: a Franka Emika Panda (7-DOF) for fixed-base manipulation and a Ridgeback omni-directional base mounted with a UR10e arm (9-DOF total) for mobile manipulation.

For the general benchmark, we generated a dataset of 50--150 unique planning queries per scenario. 
Each specific query was solved 5 times by every planner to assess repeatability and performance variance.
In addition to standard benchmarking, we incorporated a specific robustness test to evaluate solution consistency. 
In this test, each base problem instance was perturbed 10 times by applying small random noise to valid start and goal states. 
A consistent and predictable planner should yield similar performance across these perturbations.

Experiments were conducted on an Intel Core i9-12900H laptop (64GB RAM, 5.2GHz). 
All algorithms were implemented in C++ using MoveIt!~\cite{moveit} and the Flexible Collision Library (FCL)~\cite{pan2012fcl} for validity checking.
A 5-second time limit was enforced for all runs.
Aggregated results are presented below, differentiated by task type (manipulation and mobile manipulation).
Detailed breakdowns per scenario and additional experiments and ablation studies are provided in the appendix and on the project website.

\subsection{Evaluation Metrics}
\label{subsec:eval_metrics}

We assess performance using four primary metrics: \textit{success rate}, \textit{cost}, \textit{consistency}, and \textit{planning time}.
Path cost is calculated as a weighted sum of length, velocity, and acceleration:
$$ \text{cost} = L + 0.1 \cdot V + 0.01 \cdot A $$
where $L$, $V$, and $A$ represent the cumulative joint-space distance, velocity magnitude, and acceleration magnitude, respectively.
This composite metric ensures a balance between trajectory efficiency and execution smoothness.

To measure consistency, we calculate the \textit{coefficient of variation} (CV) as a percentage. 
For a set of path costs $\{c_1, c_2, \ldots, c_n\}$ obtained from $n$ repetitions, the CV is the ratio of the standard deviation to the mean ($\text{CV} = \sigma / \mu$). 
Lower CV values indicate high consistency, whereas higher values reflect significant variance in solution quality.

Finally, we report the average \textit{planning time} across all trials and the \textit{success rate}, representing the percentage of queries successfully solved within the allowed planning time.

\subsection{Evaluated Algorithms}
\label{subsec:baselines}

We compare \mgs{} against 12 baselines spanning three paradigms.
\textbf{Sampling-based:} RRT~\cite{rrt}, RRT-Connect~\cite{RRT-Connect}, BiTRRT~\cite{bitrrt}, BiEST~\cite{BiEST}, PRM~\cite{PRM}, and RRT$^*$~\cite{RRTstar}, with shortcutting (OMPL implementation) and time parameterization as postprocessing.
We exclude MTRRT~\cite{mtrrt} from the comparison: we could not find an open-source implementation, and our own implementation failed to solve most of the benchmark problems in this experimental suite.
\textbf{Optimization-based:} CHOMP~\cite{ratliff2009chomp}, STOMP~\cite{kalakrishnan2011stomp}, and a hybrid RRT-Connect+CHOMP pipeline. CHOMP and STOMP were initialized with straight-line joint-space trajectories, while RRT-Connect+CHOMP used RRT-Connect to generate an initial feasible path that was subsequently refined by CHOMP. All optimization-based planners used the MoveIt! plugin implementations.
\textbf{Search-based:} Weighted-A$^*$ (wA$^*$), MHA$^*$~\cite{aine2016mhastar}, and wPA$^*$SE~\cite{pase}, using joint-space Euclidean and workspace 3D BFS heuristics, with shortcutting (as in~\cite{srmp}) and the same time parameterization as the sampling-based methods.
For all search-based planners, we used uniform-cost edges with bounded suboptimality $w=50$\footnote{The heuristic underestimates the cost-to-go and edges have unit cost; the weight $w$ scales the heuristic to match edge transition costs and further inflates it.}.
\mgs{} has the same heuristic settings as the search-based baselines, same edge costs and suboptimality bound, and a maximum of 10 subgraphs. 

\subsection{Results and Analysis}
\label{subsec:results_analysis}

\begin{figure}[h]
    \centering
    \includegraphics[width=0.49\columnwidth]{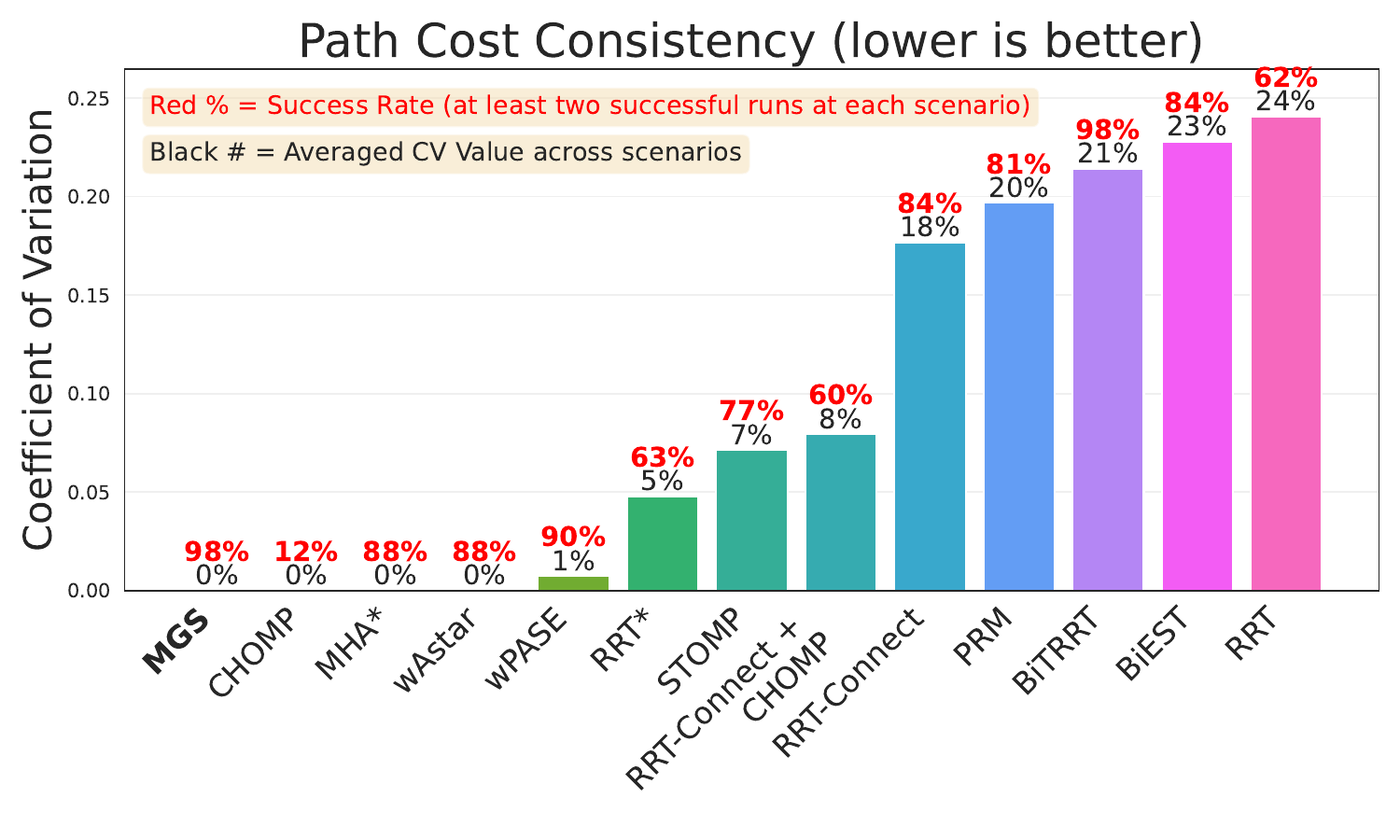}
    \includegraphics[width=0.49\columnwidth]{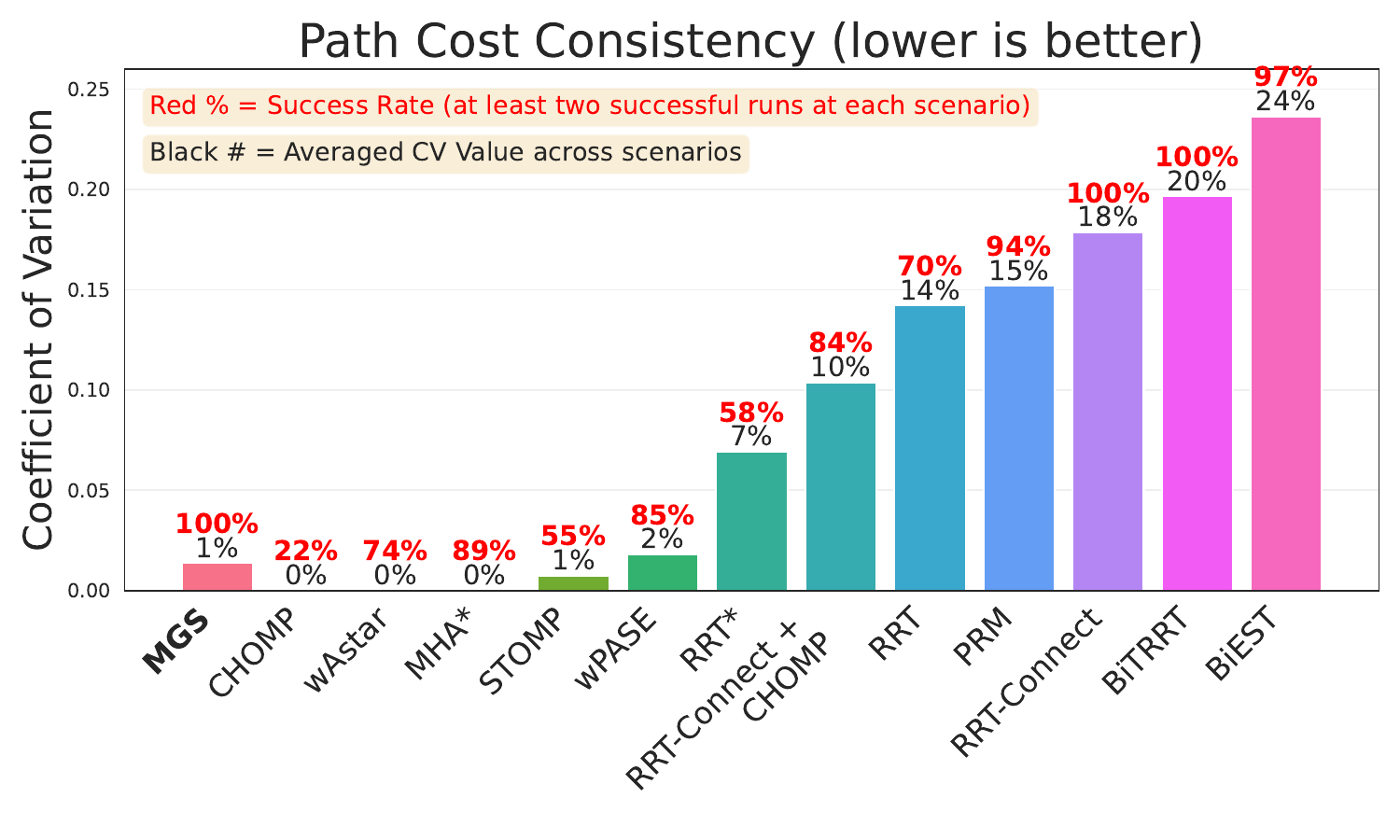}
    \caption{
        Path cost consistency for manipulation (left) and mobile manipulation (right) tasks. Each planner was executed 5 times per query; bars show the coefficient of variation (CV) of path costs across runs.
        }
    \label{fig:manipulation_consistency_results}
\end{figure}

\begin{figure}[h]
    \centering
    \includegraphics[width=0.99\columnwidth]{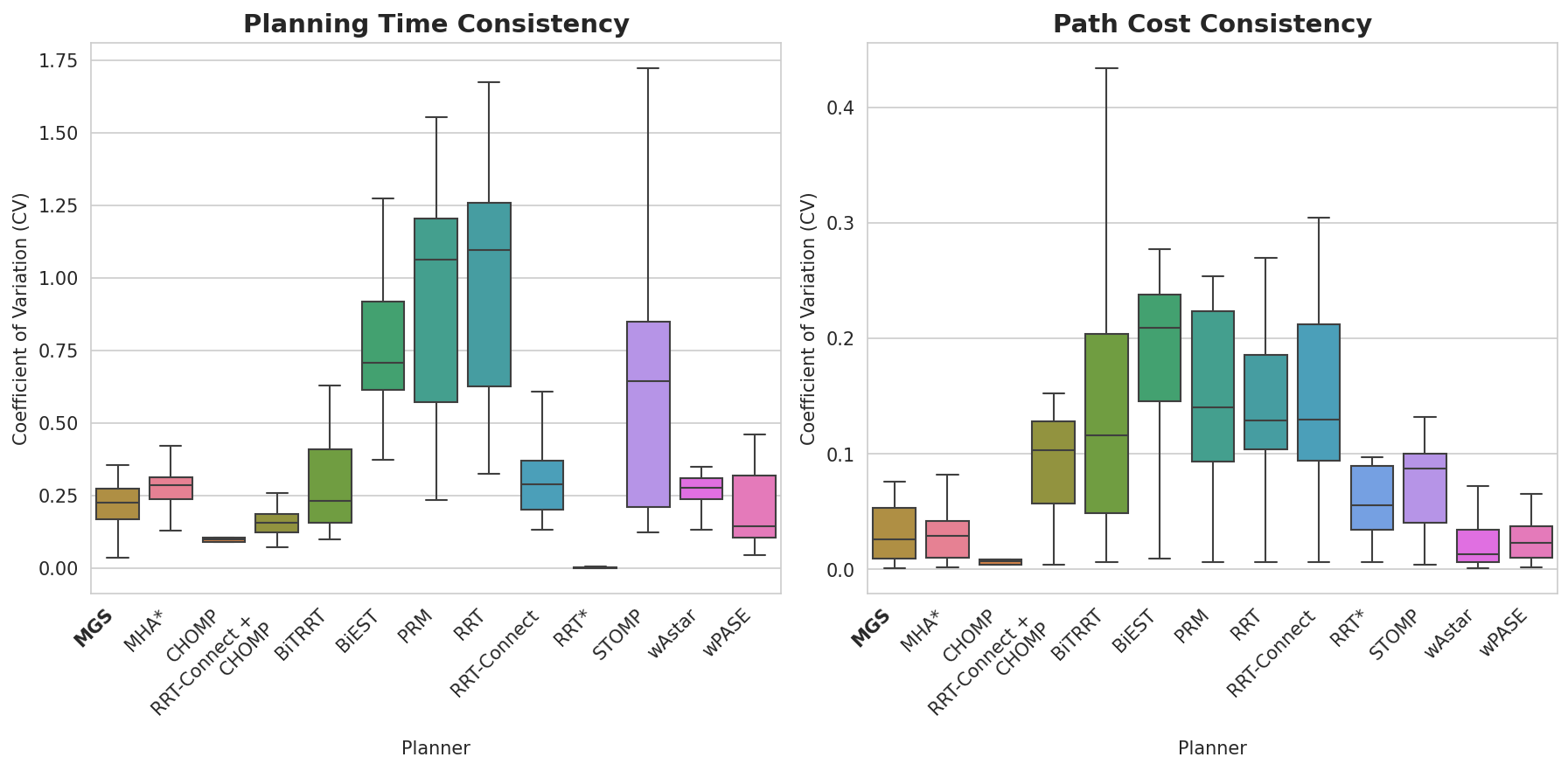}
    \includegraphics[width=0.99\columnwidth]{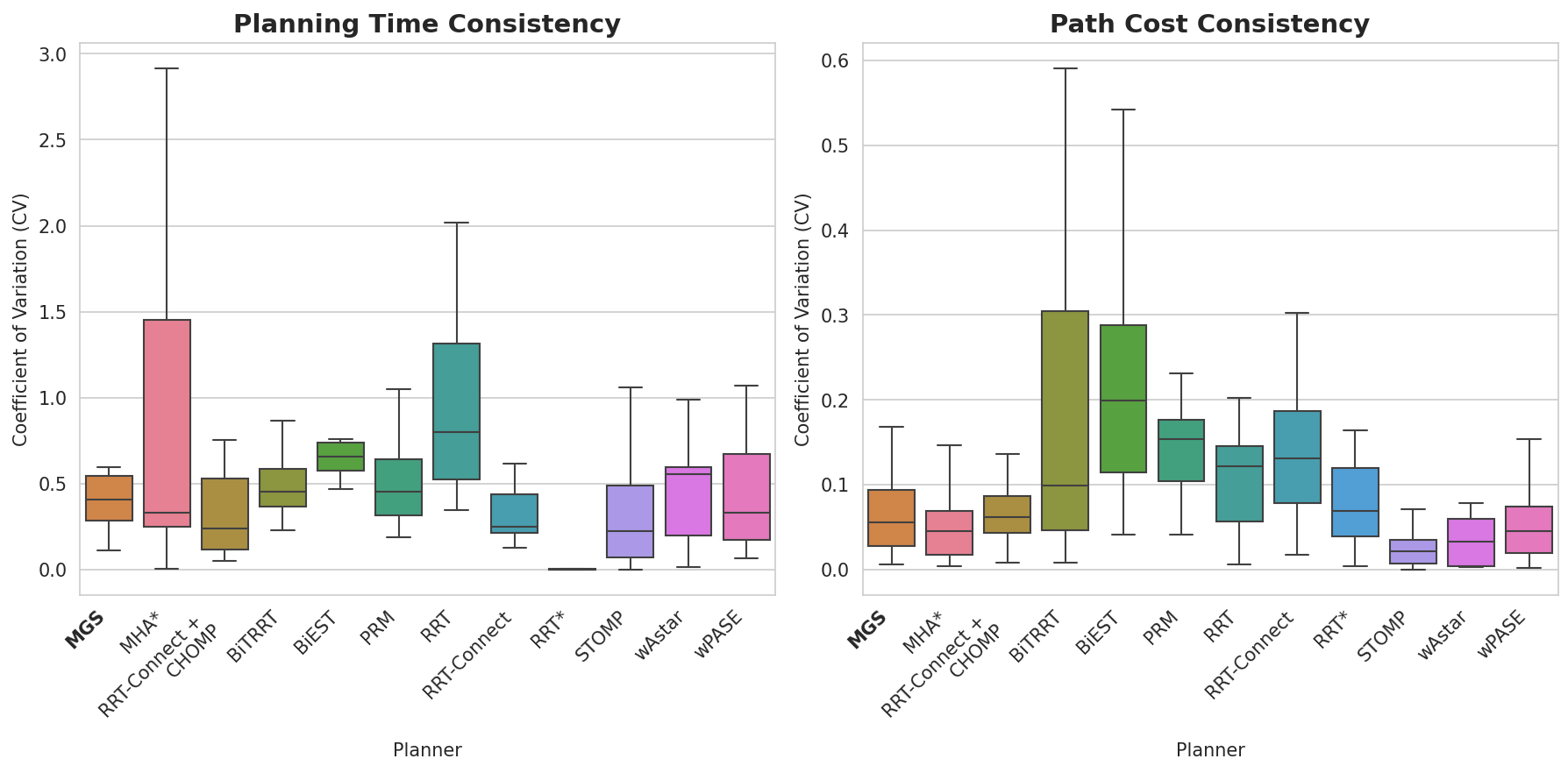}
    \caption{
        Consistency analysis under start/goal perturbations for manipulation (top) and mobile manipulation (bottom) tasks. 
        Each planner was tested on 10 perturbed versions of each base query; bars show the coefficient of variation (CV) of path costs and planning times across perturbations.
        }
    \label{fig:manipulation_consistency_results2}
\end{figure}

\begin{table*}[!ht]
\centering
\caption{
Experimental Results Summary: Aggregated performance metrics for manipulation and mobile manipulation tasks. 
Pairwise Relative Cost is the ratio of \mgs{}'s path cost to each baseline, computed only on queries where both planners succeeded.
Planning times are reported for successful runs and averaged over all runs (including timeouts).
}
\label{tab:results}
\resizebox{\textwidth}{!}{%
\begin{tabular}{|ccc|
>{\columncolor[HTML]{FFFFFF}}c |
>{\columncolor[HTML]{FFFFFF}}c |
>{\columncolor[HTML]{FFFFFF}}c |
>{\columncolor[HTML]{FFFFFF}}c |
>{\columncolor[HTML]{FFFFFF}}c |
>{\columncolor[HTML]{FFFFFF}}c |
>{\columncolor[HTML]{FFFFFF}}c |
>{\columncolor[HTML]{FFFFFF}}c |
>{\columncolor[HTML]{FFFFFF}}c |
>{\columncolor[HTML]{FFFFFF}}c |
>{\columncolor[HTML]{FFFFFF}}c |
>{\columncolor[HTML]{FFFFFF}}c |
>{\columncolor[HTML]{FFFFFF}}c |}
\hline
\multicolumn{3}{|c|}{} &
  \cellcolor[HTML]{C0C0C0}\textsc{Mgs} &
  \cellcolor[HTML]{C0C0C0}WA* &
  \cellcolor[HTML]{C0C0C0}MHA* &
  \cellcolor[HTML]{C0C0C0}wPASE &
  \cellcolor[HTML]{C0C0C0}RRT &
  \cellcolor[HTML]{C0C0C0}PRM &
  \cellcolor[HTML]{C0C0C0}RRT-Connect &
  \cellcolor[HTML]{C0C0C0}BiTRRT &
  \cellcolor[HTML]{C0C0C0}BiEST &
  \cellcolor[HTML]{C0C0C0}CHOMP &
  \cellcolor[HTML]{C0C0C0}STOMP &
  \cellcolor[HTML]{C0C0C0}\begin{tabular}[c]{@{}c@{}}RRT-Connect + \\ CHOMP\end{tabular} &
  \cellcolor[HTML]{C0C0C0}RRT* \\ \hline
\multicolumn{1}{|c|}{\cellcolor[HTML]{C0C0C0}} &
  \multicolumn{2}{c|}{Success Rate {[}\%{]}} &
  98.2 &
  87.5 &
  87.5 &
  90.4 &
  65.7 &
  81.4 &
  86.8 &
  97.5 &
  83.6 &
  12.5 &
  77.1 &
  65.7 &
  68.2 \\ \cline{2-16} 
\multicolumn{1}{|c|}{\cellcolor[HTML]{C0C0C0}} &
  \multicolumn{2}{c|}{Pairwise Relative Cost [\textsc{Mgs} / \textit{planner}]} &
  - &
  0.98 &
  0.99 &
  0.99 &
  0.74 &
  0.74 &
  0.78 &
  0.81 &
  0.77 &
  0.47 &
  0.59 &
  0.47 &
  1.11 \\ \cline{2-16} 
\multicolumn{1}{|c|}{\cellcolor[HTML]{C0C0C0}} &
  \multicolumn{1}{c|}{} &
  Successful Runs &
  0.41 &
  0.26 &
  0.14 &
  0.19 &
  0.86 &
  0.23 &
  0.37 &
  0.39 &
  0.48 &
  0.09 &
  0.25 &
  0.14 &
  5 \\ \cline{3-16} 
\multicolumn{1}{|c|}{\multirow{-4}{*}{\cellcolor[HTML]{C0C0C0}Manipulation}} &
  \multicolumn{1}{c|}{\multirow{-2}{*}{Planning Time {[}sec{]}}} &
  All Runs (including time-limit) &
  0.49 &
  1.18 &
  1.15 &
  0.94 &
  2.61 &
  1.62 &
  1.27 &
  0.56 &
  1.64 &
  4.43 &
  1.75 &
  2.19 &
  5 \\ \hline
\multicolumn{1}{|c|}{\cellcolor[HTML]{C0C0C0}} &
  \multicolumn{2}{c|}{Success Rate {[}\%{]}} &
  100 &
  74.4 &
  88.5 &
  84.9 &
  74.5 &
  94.8 &
  100 &
  100 &
  97.7 &
  22.4 &
  55.7 &
  84.6 &
  63.3 \\ \cline{2-16} 
\multicolumn{1}{|c|}{\cellcolor[HTML]{C0C0C0}} &
  \multicolumn{2}{c|}{Pairwise Relative Cost [\textsc{Mgs} / \textit{planner}]} &
  - &
  1.07 &
  1.10 &
  1.10 &
  0.82 &
  0.77 &
  0.80 &
  0.88 &
  0.72 &
  0.66 &
  0.98 &
  0.72 &
  0.90 \\ \cline{2-16} 
\multicolumn{1}{|c|}{\cellcolor[HTML]{C0C0C0}} &
  \multicolumn{1}{c|}{} &
  Successful Runs &
  0.20 &
  0.20 &
  0.32 &
  0.37 &
  0.39 &
  0.35 &
  0.11 &
  0.21 &
  0.24 &
  0.18 &
  0.49 &
  0.24 &
  5.00 \\ \cline{3-16}
\multicolumn{1}{|c|}{\multirow{-4}{*}{\cellcolor[HTML]{C0C0C0}Mobile Manipulation}} &
  \multicolumn{1}{c|}{\multirow{-2}{*}{Planning Time {[}sec{]}}} &
  All Runs (including time-limit) &
  0.20 &
  1.59 &
  0.86 &
  1.16 &
  1.48 &
  0.58 &
  0.11 &
  0.21 &
  0.46 &
  4.14 &
  2.75 &
  1.03 &
  5.00 \\ \hline
\end{tabular}%
}
\end{table*}

Table~\ref{tab:results} summarizes the aggregated results across all manipulation and mobile manipulation tasks.
\mgs{} consistently attains high success rates while producing low-cost solutions.
Among sampling-based methods, BiTRRT achieves comparable reliability but at noticeably higher path costs.
Search-based planners achieve similar solution quality to \mgs{} but suffer from lower success rates, particularly in mobile manipulation where heuristic guidance becomes less effective in higher-dimensional spaces.
Optimization-based approaches (CHOMP, STOMP) show the lowest reliability due to frequent convergence to infeasible local minima in cluttered environments.
The only method with lower costs than \mgs{} on successful queries is RRT$^*$, which is expected given its asymptotic optimality guarantee; however, this comes at the expense of significantly lower success rates and longer planning times.

Figures~\ref{fig:manipulation_consistency_results} and~\ref{fig:manipulation_consistency_results2} present the consistency analysis across repeated runs and under start/goal perturbations.
\mgs{} exhibits low coefficients of variation (CV) in path costs, indicating stable and predictable performance.
This consistency stems from the deterministic nature of the search-based anchor combined with structured root selection.
In mobile manipulation, slightly higher variability is observed, as the differential IK used for root selection (Section~\ref{subsec:root_selection}) can produce different configurations depending on the seed, leading to occasional variation in the resulting roots.
In contrast, sampling-based planners show higher variability due to stochastic exploration, where small perturbations in start/goal can lead to substantially different random trees.

\mgs{}'s performance depends on the quality and on the number of roots: too few roots can limit the benefit of multi-directional, 
while too many dilute the expansion budget across subgraphs that may not contribute to the solution.
Additional analysis of these sensitivities is provided in the appendix.
Overall, \mgs{} effectively bridges the gap between sampling-based and search-based paradigms---achieving the scalability of the former while maintaining the solution quality and consistency of the latter.

\section{Conclusion, Limitations, and Future Work}

We presented \mgs{}, a multi-graph search framework for motion planning that departs from the traditional unidirectional and bidirectional search paradigm.
By maintaining multiple subgraphs anchored at strategically chosen root configurations, \mgs{} focuses exploration on high-potential regions in the state space.
\mgs{} retains the completeness and bounded-suboptimality guarantees of classical search-based planners, while our experiments demonstrate that it achieves substantial improvements in planning efficiency.
While effective, the current approach has limitations that point to exciting future directions.
The root selection strategy reasons in end-effector workspace, implicitly assuming that the end-effector position is a sufficient proxy for identifying important regions in configuration-space, and tends to produce roots that are close to obstacle boundaries.
For tasks with different constraints---such as preferring greater clearance---the workspace BFS may miss critical regions.
Since \mgs{} is agnostic to the root selection strategy, these limitations can be addressed by plugging in alternative strategies---such as curating libraries of kinematically favorable configurations, reasoning about the full robot body, or learning root placement policies from data.
Dynamically selecting which subgraphs to expand during search is another natural extension.
More broadly, integrating learned motion primitives or manipulation skills as subgraph components could extend \mgs{} beyond collision-free planning to contact-rich tasks, and exploiting the multi-graph structure for parallelization could yield further speedups.


\bibliographystyle{IEEEtran}
\bibliography{references}

\appendix

We first give additional algorithmic details, including pseudocode, then discuss attractors and their geometric role in guiding multi-graph search, and finally summarize implementation details.
We also report additional experimental results: per-scenario breakdowns of Section~\ref{sec:experiments} and ablations on the time-limit and the maximum number of sub-graphs parameters.
Finally, we present an additional domain demonstrating the generalization of \mgs{}, discuss the impact of joint limits relative to baselines, and analyze failure cases.

\subsection{Algorithmic Details}
\label{subsec:appendix_algo}

We present pseudocode for \textsc{MergeSubGraphs} (Algorithm~\ref{alg:merge}) and summarize \textsc{TryToConnect}, \textsc{GetConnectingPath}, and \textsc{ChooseMergingOrder} below.

\textsc{TryToConnect} checks whether a state $q$, expanded in a sub-graph, can be connected to states that already belong to other sub-graphs.
Any method that solves the two-point boundary value problem between $q$ and states in the target sub-graphs can be used.
Here we use linear interpolation in configuration space: for each nearest neighbor $q'$ on the frontier of a target sub-graph $G_j$, the interpolated path from $q$ to $q'$ is collision-checked.
If the path is collision-free, the connection $(q', G_j)$ is recorded.
\textsc{GetConnectingPath} returns the corresponding collision-free interpolated path.

\begin{algorithm}[h]
    \SetAlgoVlined
    \SetKwFunction{OPEN}{OPEN}
    \SetKwFunction{CLOSED}{CLOSED}
    \SetKwFunction{FOCAL}{FOCAL}
    \SetKwFunction{INSERT}{Insert}
    \SetKwFunction{AddEdge}{AddEdge}

    \SetKwInput{KwIn}{\textsc{Input}}
    \SetKwInput{KwOut}{\textsc{Output}}

    \scriptsize
    \caption{\textsc{MergeSubGraphs}}
    \label{alg:merge}
    \KwIn{Receiving sub-graph $G_{\text{from}}$, merged sub-graph $G_{\text{to}}$, merge point $q_{\text{merge}} \in V_{\text{to}}$, sub-graph collection $\mathcal{G}$}
    \KwOut{Updated collection $\mathcal{G}$}
    \vspace{4pt}

    \tcp{BFS from merge point: propagate g-values and transfer edges/states}
    $Q \gets \{q_{\text{merge}}\}$
    $\text{visited} \gets \{q_{\text{merge}}\}$

    \While{$Q \neq \emptyset$}{
        $s \gets Q.\text{dequeue}()$

        \ForEach{edge $(s, s', c(s,s'))$ from $s$ in $G_{\text{to}}$}{ \label{line:edges}

            \If{$s' \notin \text{visited}$}{
                $g(s') \gets G_{from}.getGvalue(s) + c(s, s')$ \tcp*[r]{\color{orange} \scriptsize Propagate from $G_{\text{to}}$'s g-values}

                $f(s') \gets g(s') + h(s')$

                $\text{visited} \gets \text{visited} \cup \{s'\}$

                \textsc{AddStateTo}($G_{\text{from}}$, $s'$)

                $Q.\text{enqueue}(s')$
            }
            $G_{\text{from}}$.\AddEdge{$(s, s', c(s,s'))$} 

        }
    }

    $\mathcal{G} \gets \mathcal{G} \setminus \{G_{\text{to}}\}$

    \KwRet $\mathcal{G}$

    \vspace{4pt}
    \SetKwProg{Fn}{Function}{:}{}
    \Fn{\textsc{AddStateTo}($G$, $s$)}{ \label{line:anchor_check}
        \If{$s \in V_G$}{ \label{line:dup_check}
            \KwRet
        }
        \eIf{$G = G_1$}{
            \tcp{Anchor merge: state enters \texttt{OPEN} for re-expansion}
            $G$.\OPEN{}.\INSERT{$s$}

            \If{$f(s) \leq \epsilon \cdot f_{\min}$}{
                $G$.\FOCAL{}.\INSERT{$s$}
            }
        }{
            \tcp{Connect-connect merge: preserve closed status}
            \eIf{$s \in G_{\text{to}}.\CLOSED()$}{
                $G$.\CLOSED{}.\INSERT{$s$}
            }{
                $G$.\OPEN{}.\INSERT{$s$}
            }
        }
    }
\end{algorithm}

\textsc{MergeSubGraphs} (Algorithm~\ref{alg:merge}) transfers all states and edges from $G_{\text{to}}$ into $G_{\text{from}}$.
Starting from the merge point $q_{\text{merge}}$, a BFS traverses $G_{\text{to}}$ using the stored edges and their costs (Line~\ref{line:edges}).
For each state $s'$ adjacent to the current state $s$, the g-value is computed using the g-value of $s$ in $G_{\text{from}}$ plus the edge cost.
Each edge is transferred to $G_{\text{from}}$, and each newly visited state is added via \textsc{AddStateTo} (Line~\ref{line:anchor_check}) if not already present (Line~\ref{line:dup_check}).
When merging into the anchor $G_1$, all states are inserted into \texttt{OPEN} to allow re-expansion under the anchor's heuristic ordering.
When merging two connect searches, closed states remain closed, avoiding redundant expansions until the eventual merge into the anchor.

\textsc{ChooseMergingOrder}($G_i$, $G_j$, $\mathcal{G}$) determines which sub-graph absorbs the other during a connect-connect merge.
If either sub-graph is the anchor $G_1$, the anchor always receives.
Otherwise, the sub-graph whose root has a smaller admissible heuristic estimate $h(r, q_{\text{start}})$---estimated to be closer to the start---is designated as $G_{\text{from}}$, as it is more likely to merge into the anchor sooner.

\subsection{Backward and Forward Attractors}
\label{subsec:appendix_attractors}

\begin{figure*}[t]
    \centering
    \begin{subfigure}[b]{0.48\textwidth}
        \centering
        \includegraphics[width=\columnwidth]{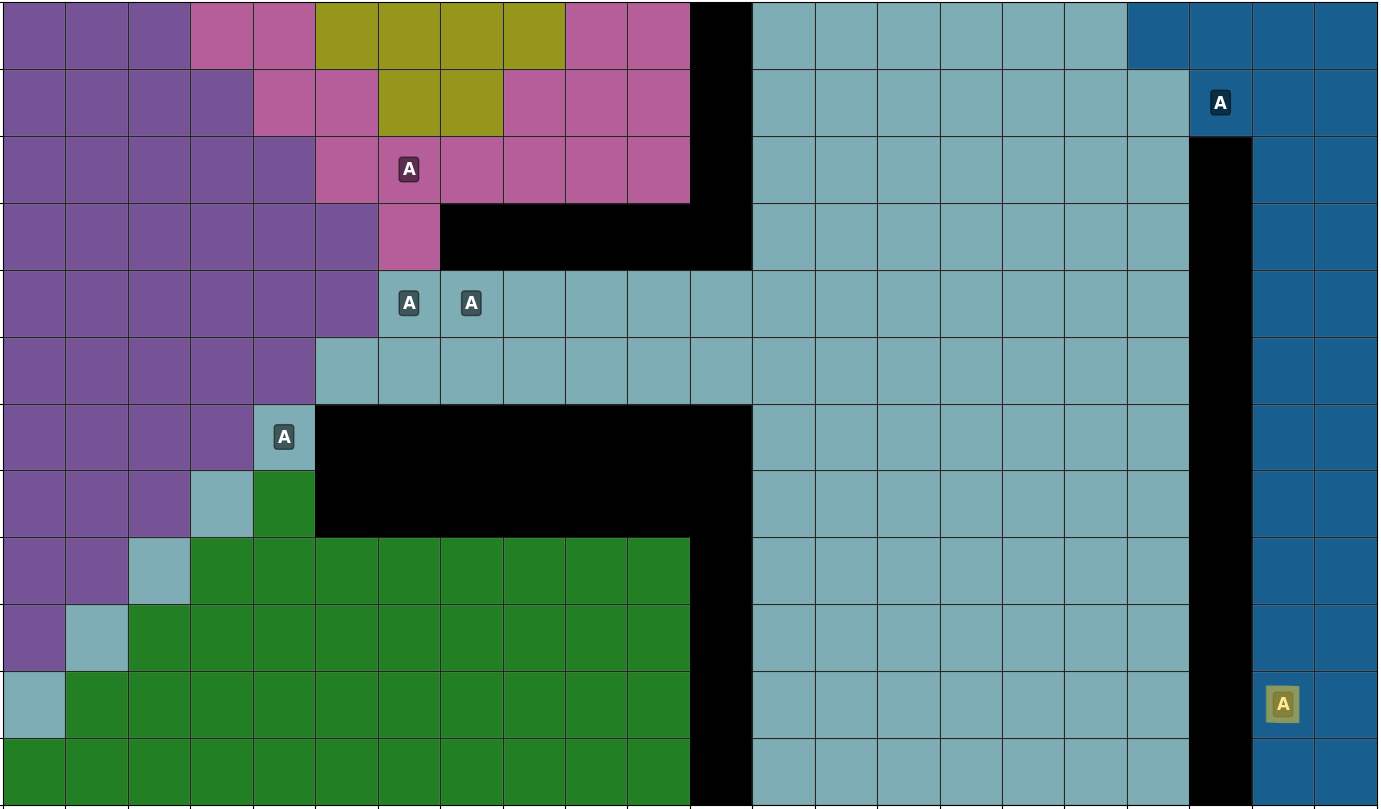}
        \caption{Backward attractors from the goal.}
        \label{fig:attractors_back}
    \end{subfigure}
    \hfill
    \begin{subfigure}[b]{0.48\textwidth}
        \centering
        \includegraphics[width=\columnwidth]{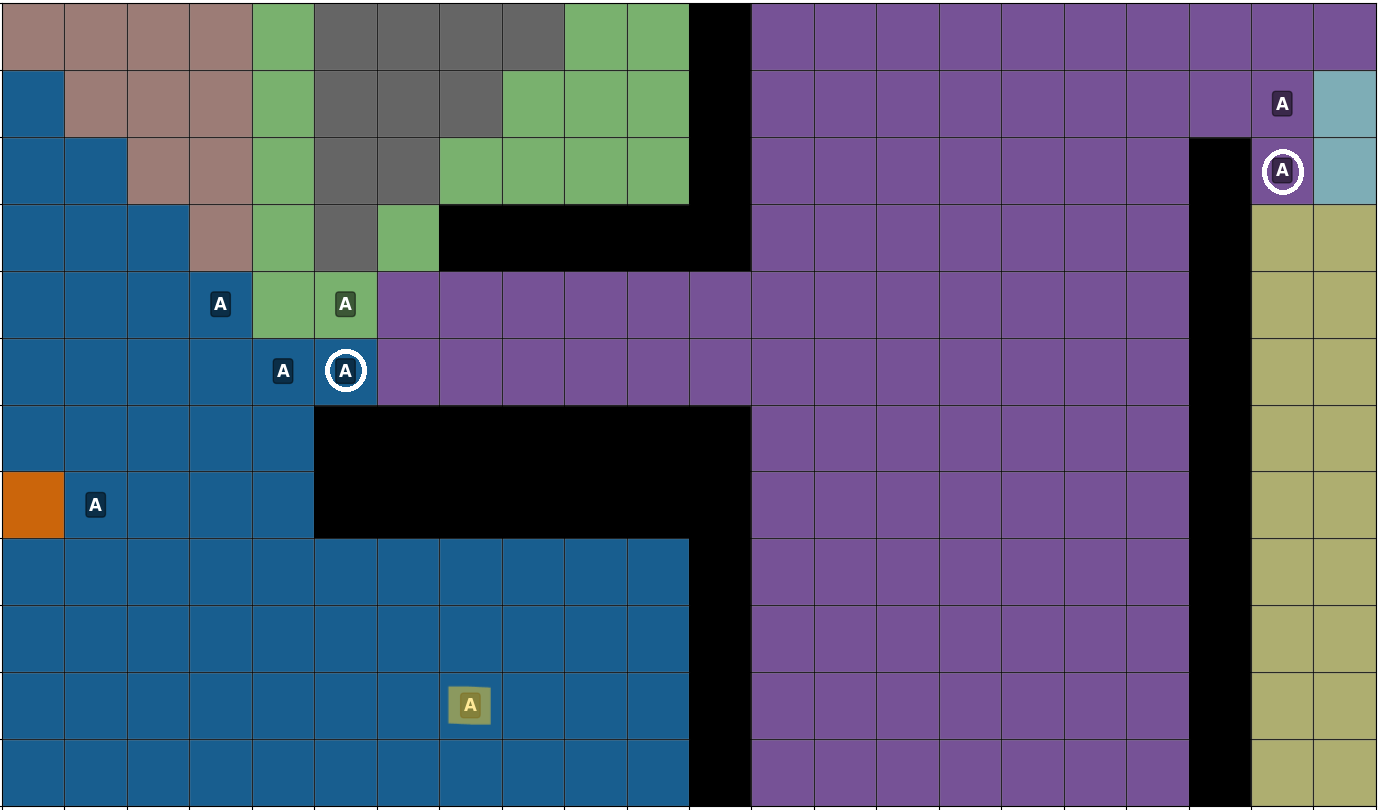}
        \caption{Forward attractors from the start.}
        \label{fig:attractors_front}
    \end{subfigure}
    \caption{Comparison of backward and forward attractor discovery. Backward (left) attractors are often generated facing away from the goal, while forward (right) attractors face toward the goal. The attractors discovered by the forward policy rollout (not the full BFS) are marked with a white circle. The root of the search is highlighted in yellow.}
    \label{fig:attractors_back_front}
\end{figure*}

The root selection procedure (Algorithm~\ref{alg:root_selection}) identifies attractor states in the end-effector workspace through backward BFS from the goal position.
Here we provide further intuition on the geometric role of these attractors and the distinction between backward and forward attractors.
As the BFS wavefront propagates outward from the goal, it encounters obstacles and flows around them.
An attractor emerges at a location $w$ where the greedy path from $w$ back toward the goal must diverge due to an obstacle: the greedy predecessor of $w$'s neighbor differs from the BFS parent (Algorithm~\ref{alg:root_selection}, Line~\ref{line:greedy_pred}), indicating that the obstacle geometry forces a detour.
Geometrically, backward attractors tend to form on the side of obstacles \emph{facing away from the goal}, effectively ``hugging'' the obstacle from the side opposite the BFS root.
Each attractor thus marks the entrance to a region that is not directly reachable from the goal without navigating around an obstacle, making it a natural candidate for a search root that can explore that region locally.

While backward attractors provide broad coverage of obstacle boundaries, they may be hard to connect to in full configuration space.
To ensure coverage along the actual start-to-goal corridor, Algorithm~\ref{alg:root_selection} traces the BFS policy \emph{forward} from the start end-effector position (Line~\ref{line:forward_attractors}), collecting the attractors encountered along this path.
This forward rollout ensures obstacle sides are represented with respect to \emph{both} the start and the goal, yielding roots that are both broadly placed near critical boundaries and concentrated along the relevant corridor (Fig.~\ref{fig:attractors_back_front}).

\subsection{Implementation Details}
\label{subsec:appendix_impl}

We describe additional implementation choices not fully specified in the main text.

\noindent\textbf{Motion Primitives and Discretization.}
The implicit graph is constructed using single-joint motion primitives: for each joint $j \in \{1, \ldots, d\}$, the primitive applies a displacement of $\pm \Delta\theta_j$ to joint $j$ while holding all other joints fixed.
We use adaptive motion primitive---long and short---where if the current state is within a certain distance threshold of the goal or start, we use both short and long primitives, otherwise only long primitives.
All edges have uniform cost, consistent with the search-based baselines described in Section~\ref{subsec:baselines}.
The joint resolution $\Delta\theta_j$ may vary per joint---typically finer for wrist joints and coarser for shoulder or base joints---and is set to match the SRMP framework~\cite{srmp} used in our experiments with a resolution of 1 degree for revolute joints and 1 cm for prismatic joints.
Please note that while we discretize for planning, the resulting solutions are in continuous space and the goal reached is accurate (the state datastructure contains both the discretized configuration and the underlying continuous configuration).

\noindent\textbf{Root Selection.}
The end-effector workspace is discretized into a 3D occupancy grid, where each voxel is marked as occupied if it overlaps with any obstacle geometry.
We use a 2 cm voxel size for manipulation and a 5 cm voxel size for mobile manipulation, and inflate obstacles by 5 cm based on the gripper geometry.
This provides conservative clearance during workspace reasoning: the signed distance field (SDF) derived from the occupancy grid is inflated around the end-effector, ensuring that identified attractors maintain a margin from obstacle surfaces.
For mapping workspace attractors to configuration space, we use differential inverse kinematics seeded with the configuration of the previously mapped attractor (starting from the start configuration).
If the IK solver fails to converge or returns a configuration in collision, the attractor is discarded.
When the number of attractors exceeds the sub-graph budget $m$, $k$-means clustering is applied in workspace coordinates, and the attractor closest to each cluster centroid is selected as the representative root.

\noindent\textbf{Sub-graph Connections.}
For \textsc{TryToConnect}, we use a single nearest neighbor ($k=1$) on the frontier of each target sub-graph when attempting connections.

\begin{figure*}[t]
    \centering
    \begin{subfigure}[b]{0.32\textwidth}
        \centering
        \includegraphics[width=\columnwidth]{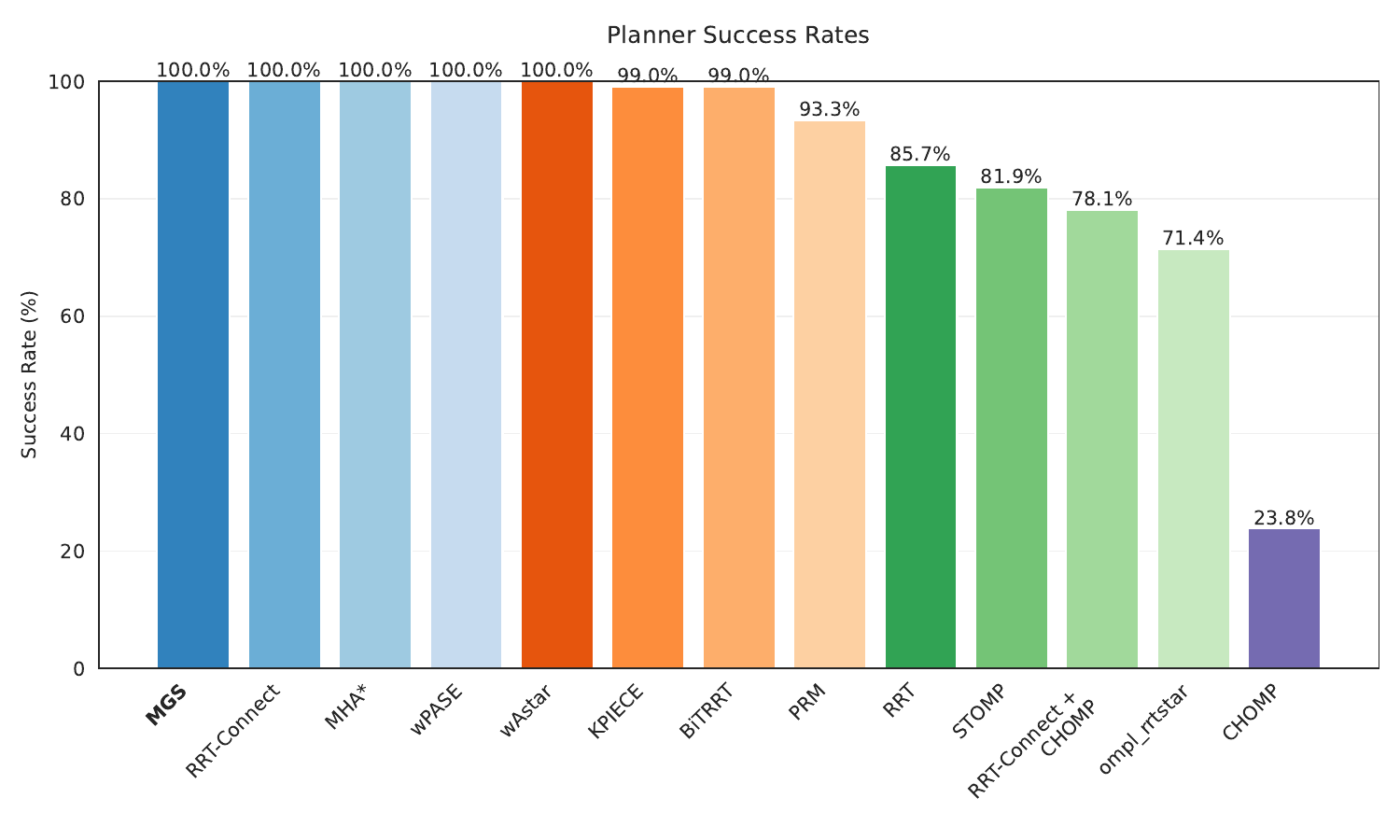}
        \caption{Shelf pick-and-place (manip.).}
    \end{subfigure}
    \hfill
    \begin{subfigure}[b]{0.32\textwidth}
        \centering
        \includegraphics[width=\columnwidth]{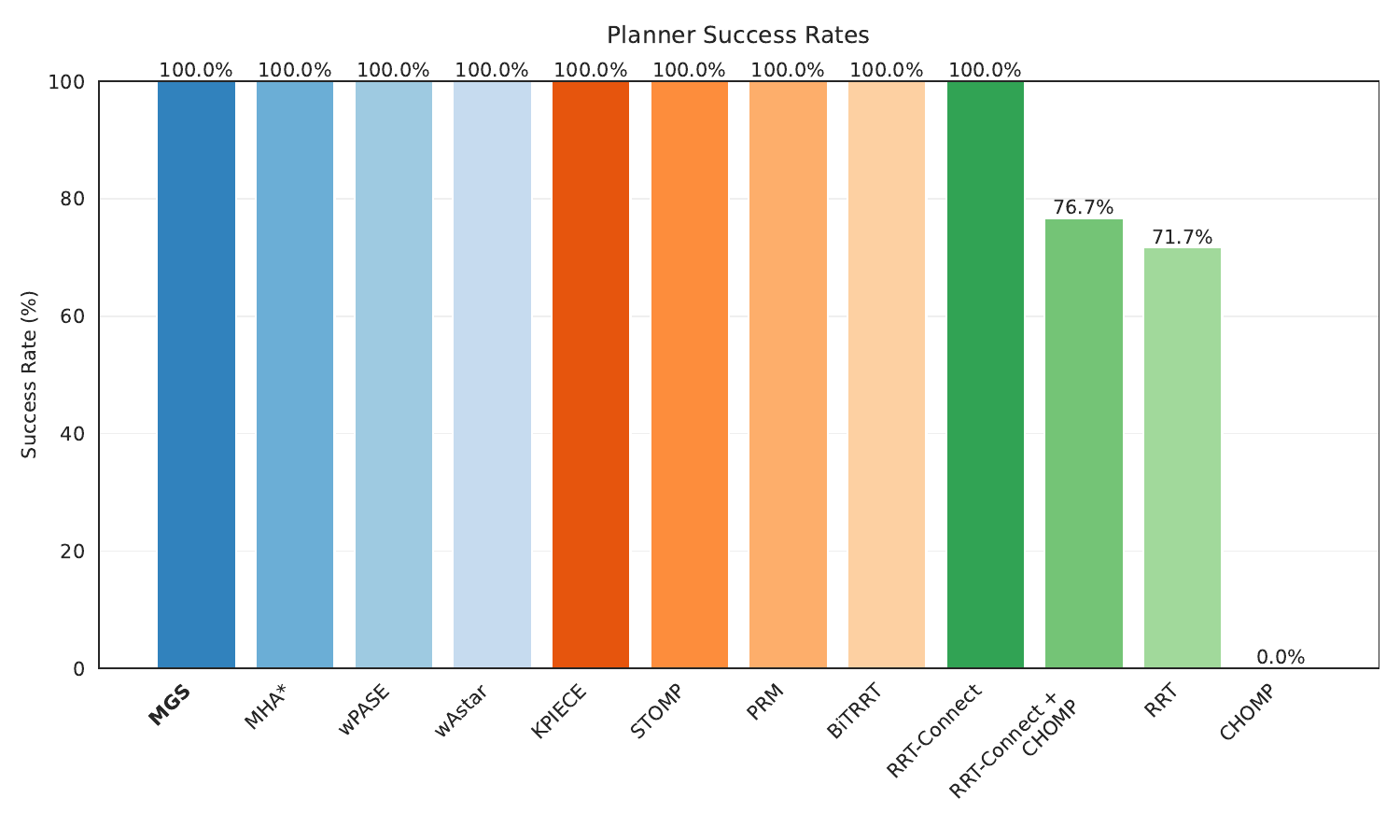}
        \caption{Bin picking (manip.).}
    \end{subfigure}
    \hfill
    \begin{subfigure}[b]{0.32\textwidth}
        \centering
        \includegraphics[width=\columnwidth]{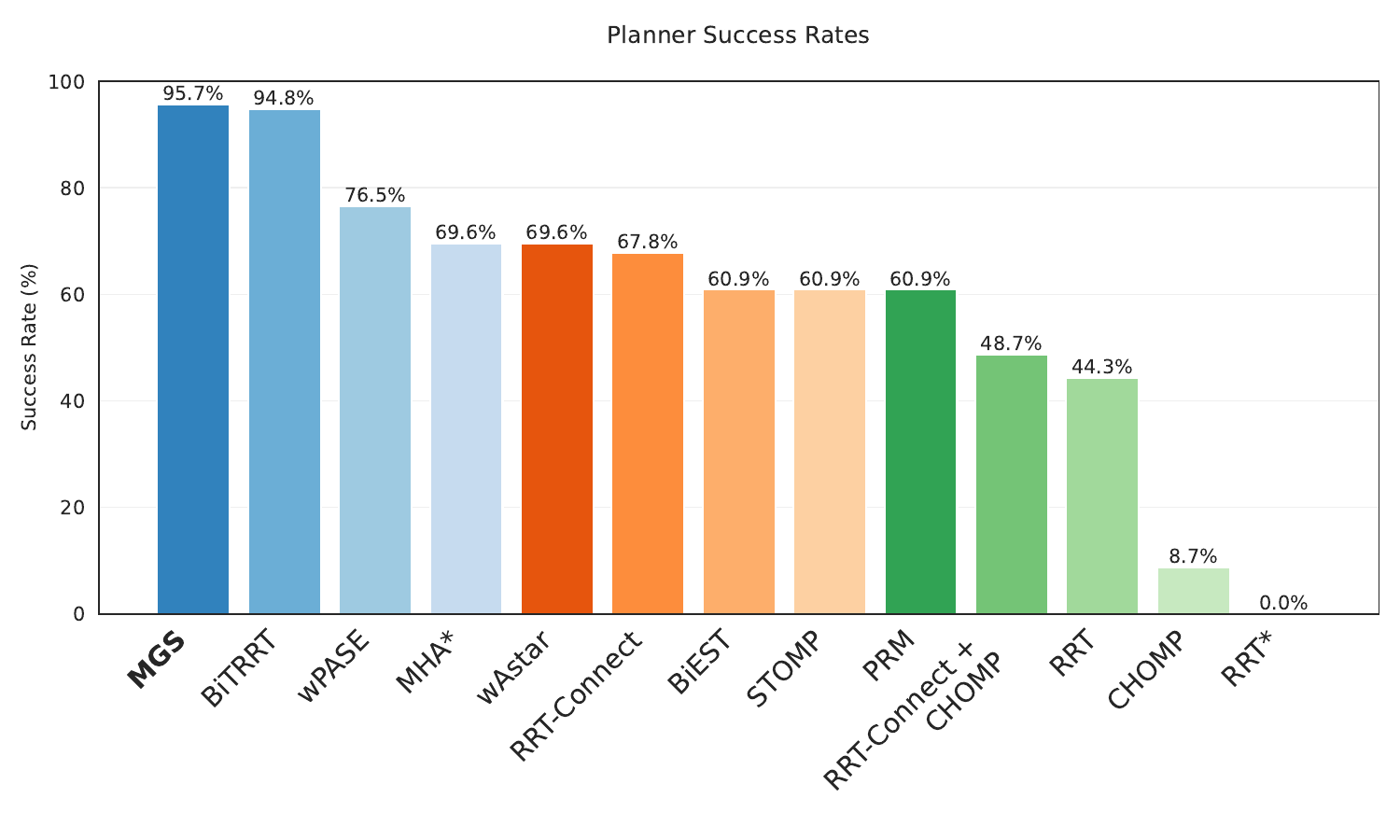}
        \caption{Cage extraction (manip.).}
    \end{subfigure}
    \\[0.5em]
    \begin{subfigure}[b]{0.24\textwidth}
        \centering
        \includegraphics[width=\columnwidth]{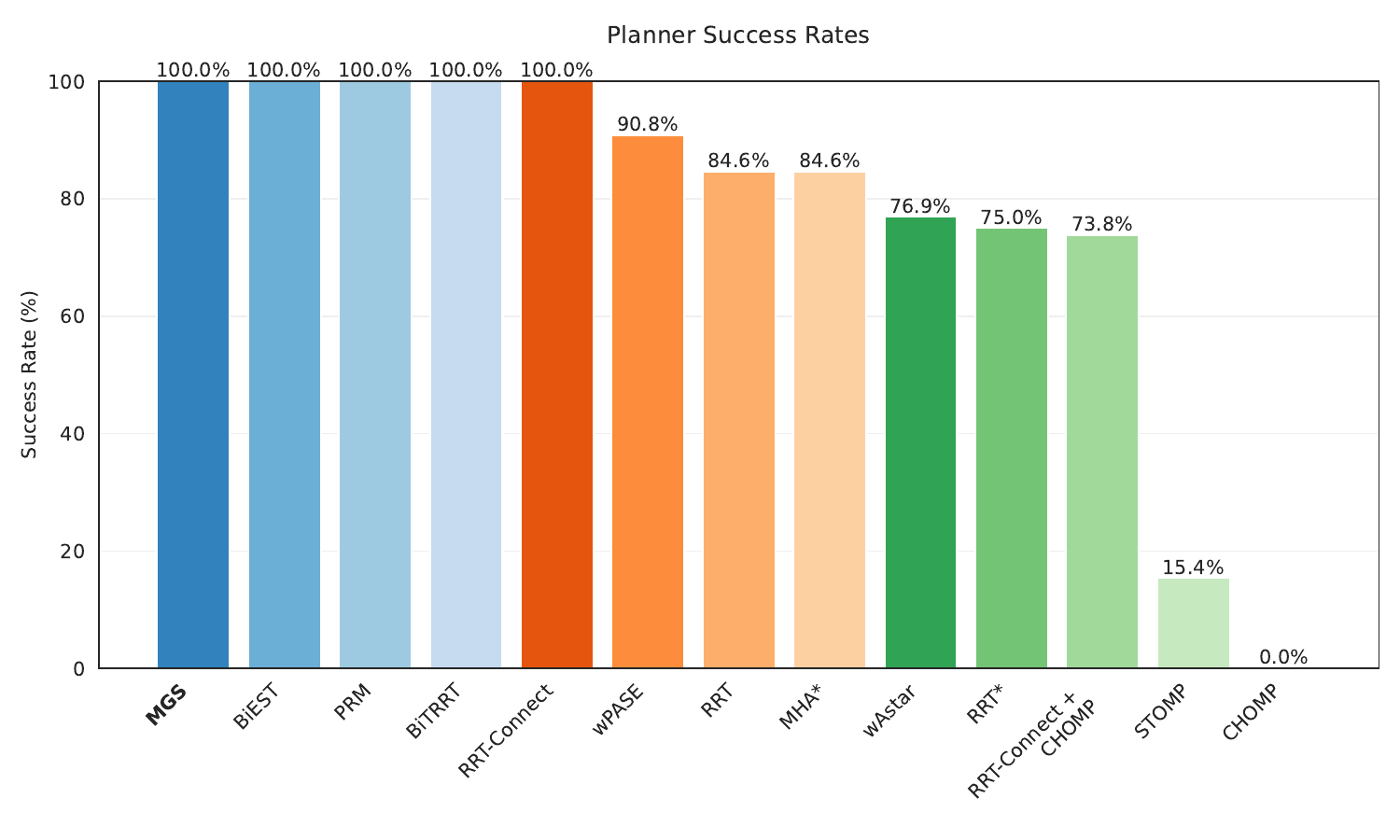}
        \caption{Low-clearance (mobile).}
    \end{subfigure}
    \hfill
    \begin{subfigure}[b]{0.24\textwidth}
        \centering
        \includegraphics[width=\columnwidth]{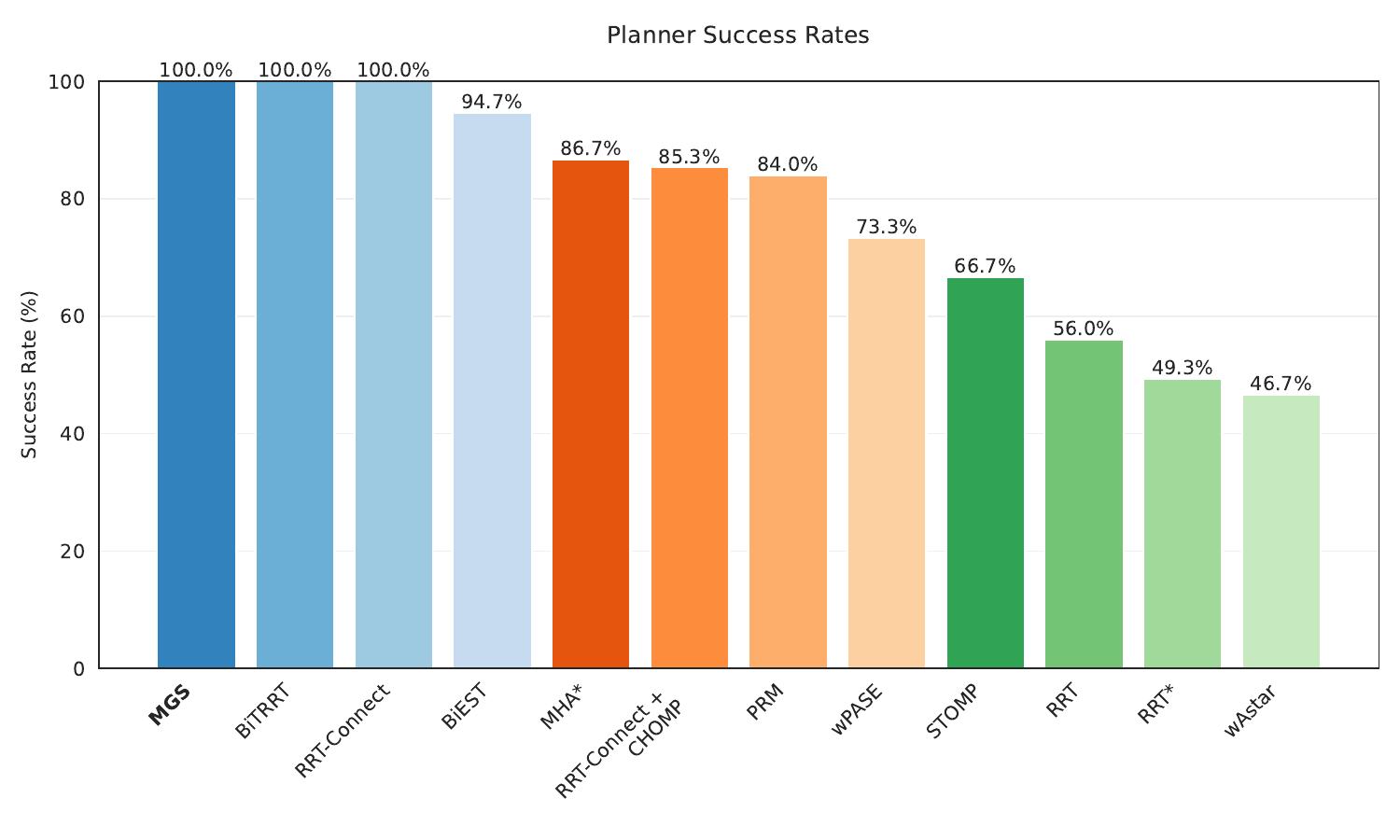}
        \caption{Deep shelf (mobile).}
    \end{subfigure}
    \hfill
    \begin{subfigure}[b]{0.24\textwidth}
        \centering
        \includegraphics[width=\columnwidth]{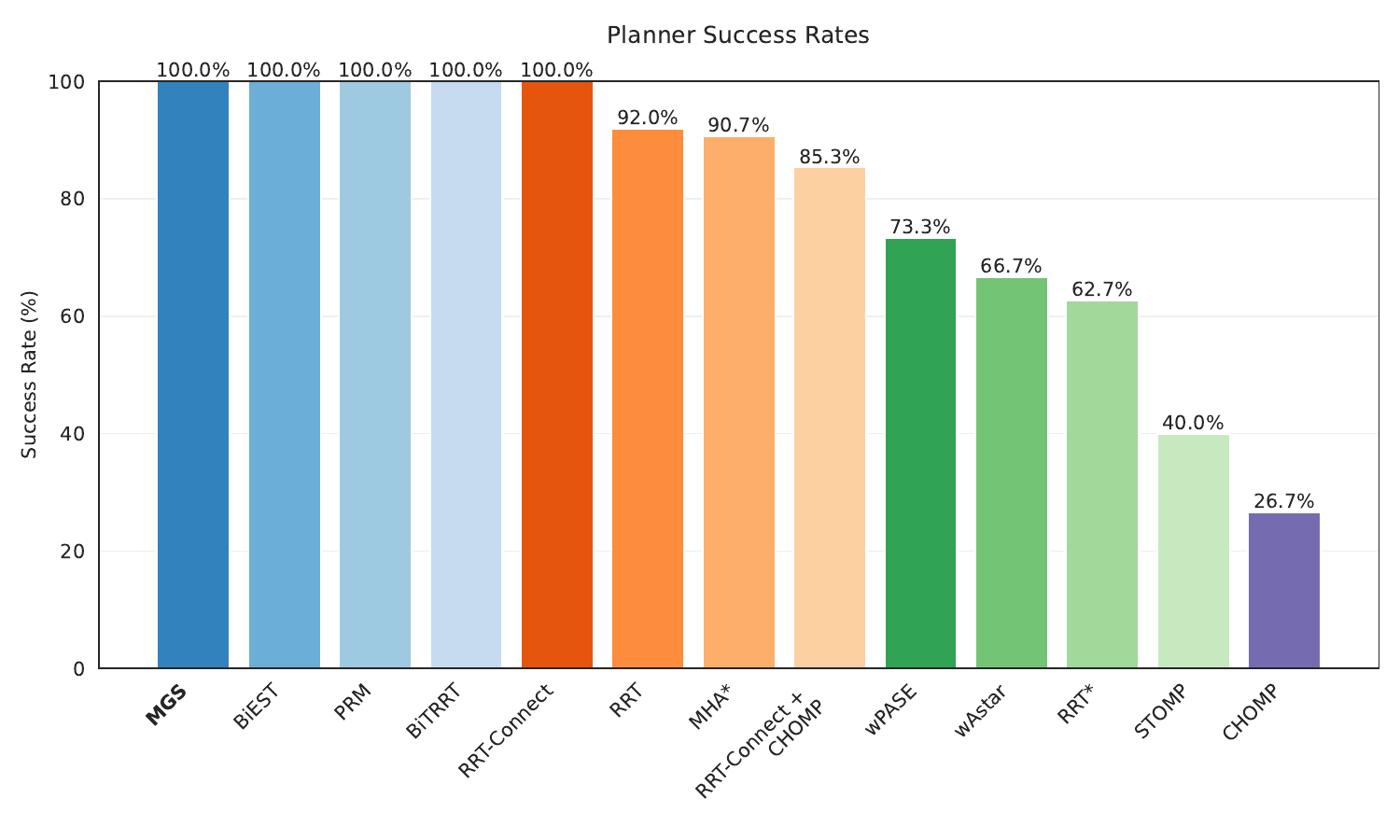}
        \caption{Cluttered table (mobile).}
    \end{subfigure}
    \hfill
    \begin{subfigure}[b]{0.24\textwidth}
        \centering
        \includegraphics[width=\columnwidth]{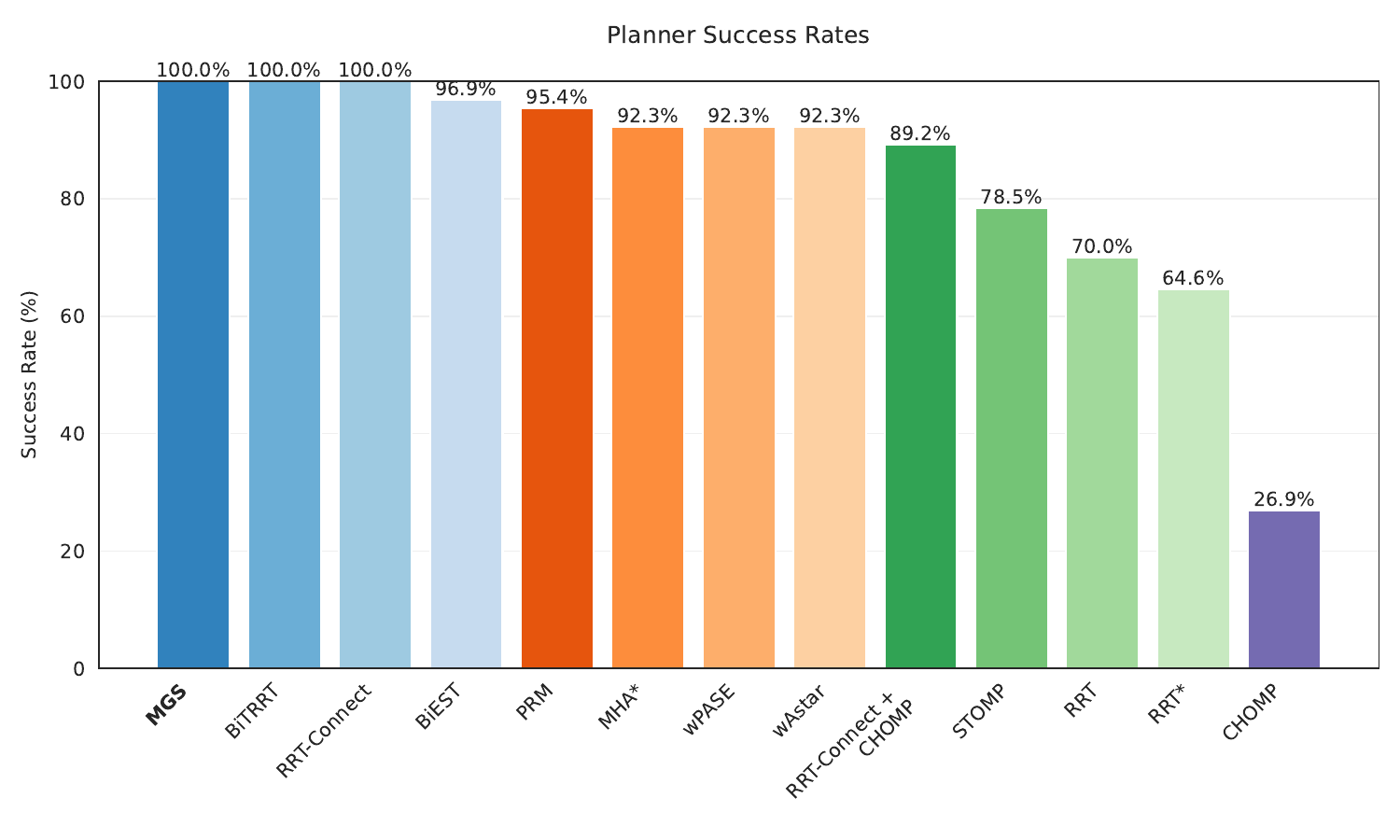}
        \caption{Warehouse (mobile).}
    \end{subfigure}
    \caption{Success rates for all scenarios. Manipulation scenarios (a--c) use a 7-DOF Franka Panda; mobile manipulation scenarios (d--g) use a 9-DOF Ridgeback + UR10e. Each bar represents the percentage of queries solved within the 5-second time limit.}
    \label{fig:per_scenario_succ}
\end{figure*}

\subsection{Per-Scenario Experimental Results}
\label{subsec:appendix_per_scenario}

We provide detailed per-scenario breakdowns of the experimental results summarized in Section~\ref{sec:experiments}.
We show success rates for each scenario (Fig.~\ref{fig:per_scenario_succ}) and pairwise cost comparison matrices (Figs.~\ref{fig:per_scenario_conf_manip} and~\ref{fig:per_scenario_conf_mobile}) for manipulation and mobile manipulation scenarios, respectively.

\subsection{Time Limit Ablation}
\label{subsec:appendix_time_limit}

\begin{figure}[t]
    \centering
    \begin{subfigure}{\columnwidth}
        \centering
        \includegraphics[width=\columnwidth]{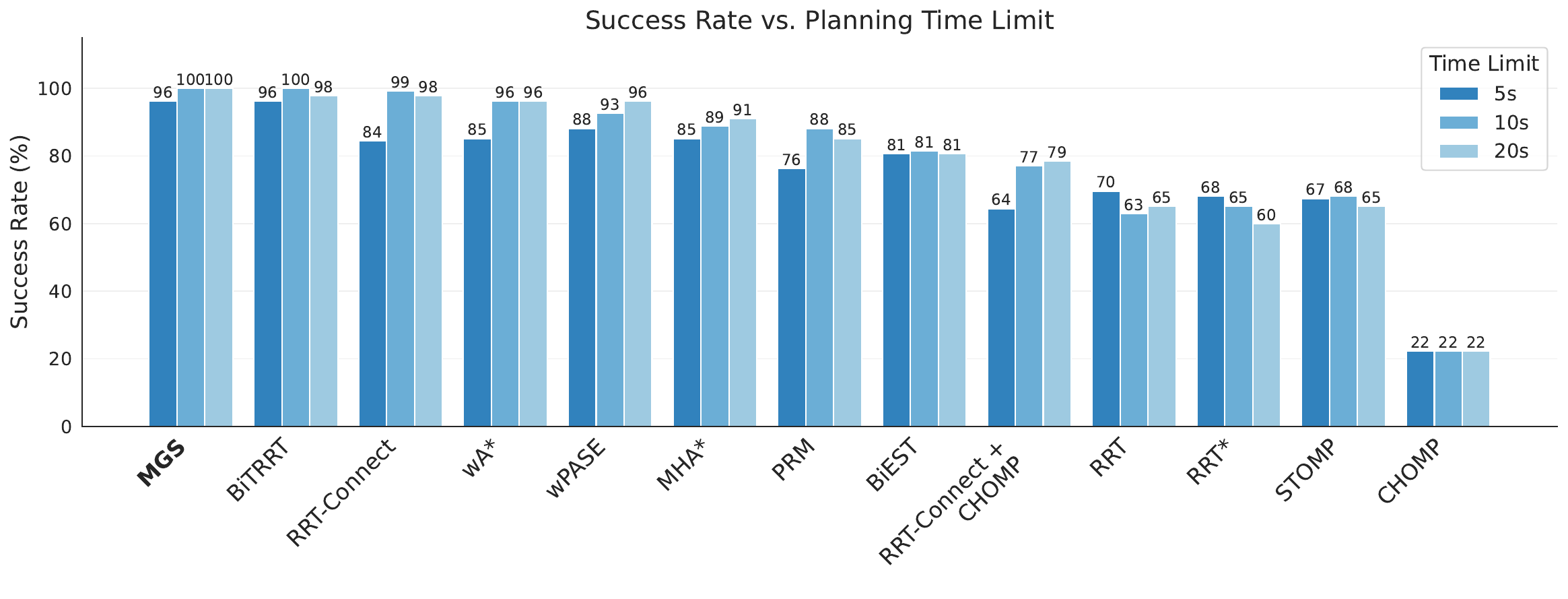}
        \caption{Success rates for varying time limits.}
    \end{subfigure}
    \\[0.5em]
    \begin{subfigure}{\columnwidth}
        \centering
        \includegraphics[width=\columnwidth]{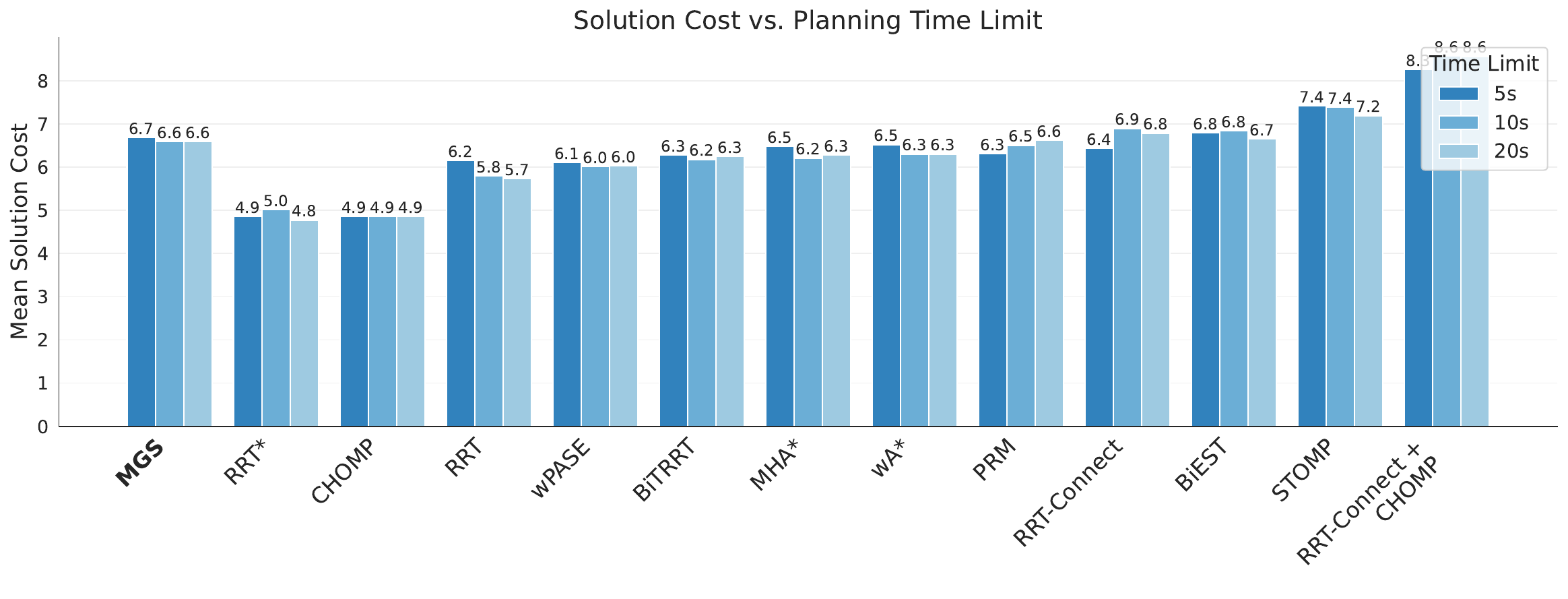}
        \caption{Solution costs for varying time limits.}
    \end{subfigure}
    \caption{Ablation on the planning time limit for time limits of 5s, 10s, and 20s.}
    \label{fig:time_limit_ablation}
\end{figure}

We study the impact of the planning time limit on performance.
Fig.~\ref{fig:time_limit_ablation} shows success rates and solution costs for varying time limits (5s, 10s, 20s) in the shelf pick-and-place scenario and the cage extraction scenario.
The relative performance between planners remains consistent across time limits.
Notably, sampling-based planners occasionally exhibit decreased success rates at larger time limits (e.g., 10s to 20s); this counterintuitive behavior stems from statistical variance inherent to their randomized nature, unlike search-based planners which are deterministic and show monotonic improvement with increased time.

\begin{figure}[h!]
    \centering
    \includegraphics[width=\columnwidth]{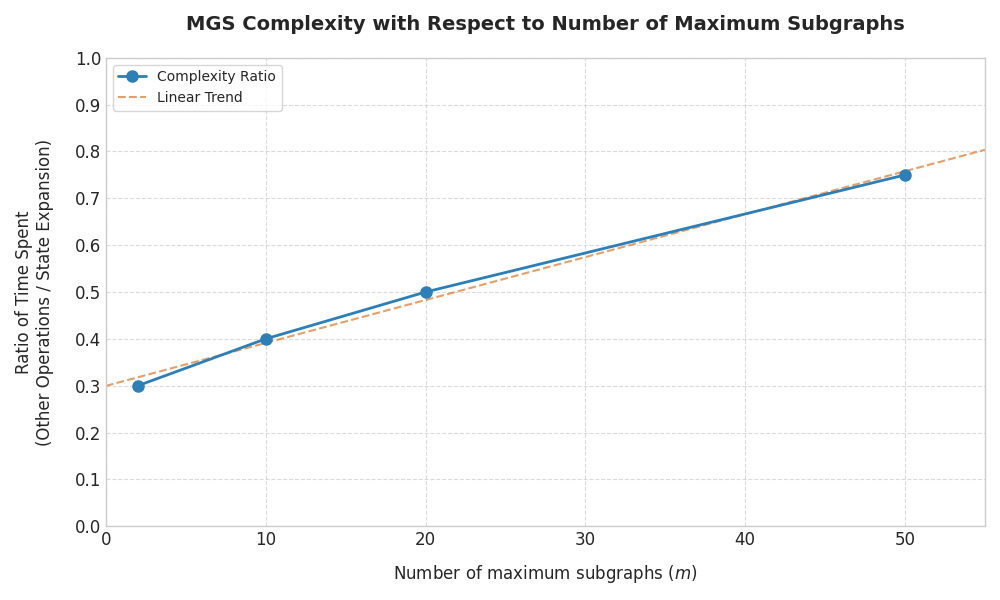}
    \caption{Computational overhead as a function of the number of sub-graphs. The y-axis shows the ratio of time spent in auxiliary operations (\textsc{TryToConnect}, \textsc{MergeSubGraphs}, sub-graph management) to time spent in state expansion. As the number of sub-graphs increases, the overhead grows approximately linearly.}
    \label{fig:complexity_vs_subgraphs}
\end{figure}

\begin{figure*}[htbp]
    \centering
    \begin{subfigure}[b]{0.48\textwidth}
        \centering
        \includegraphics[width=\columnwidth]{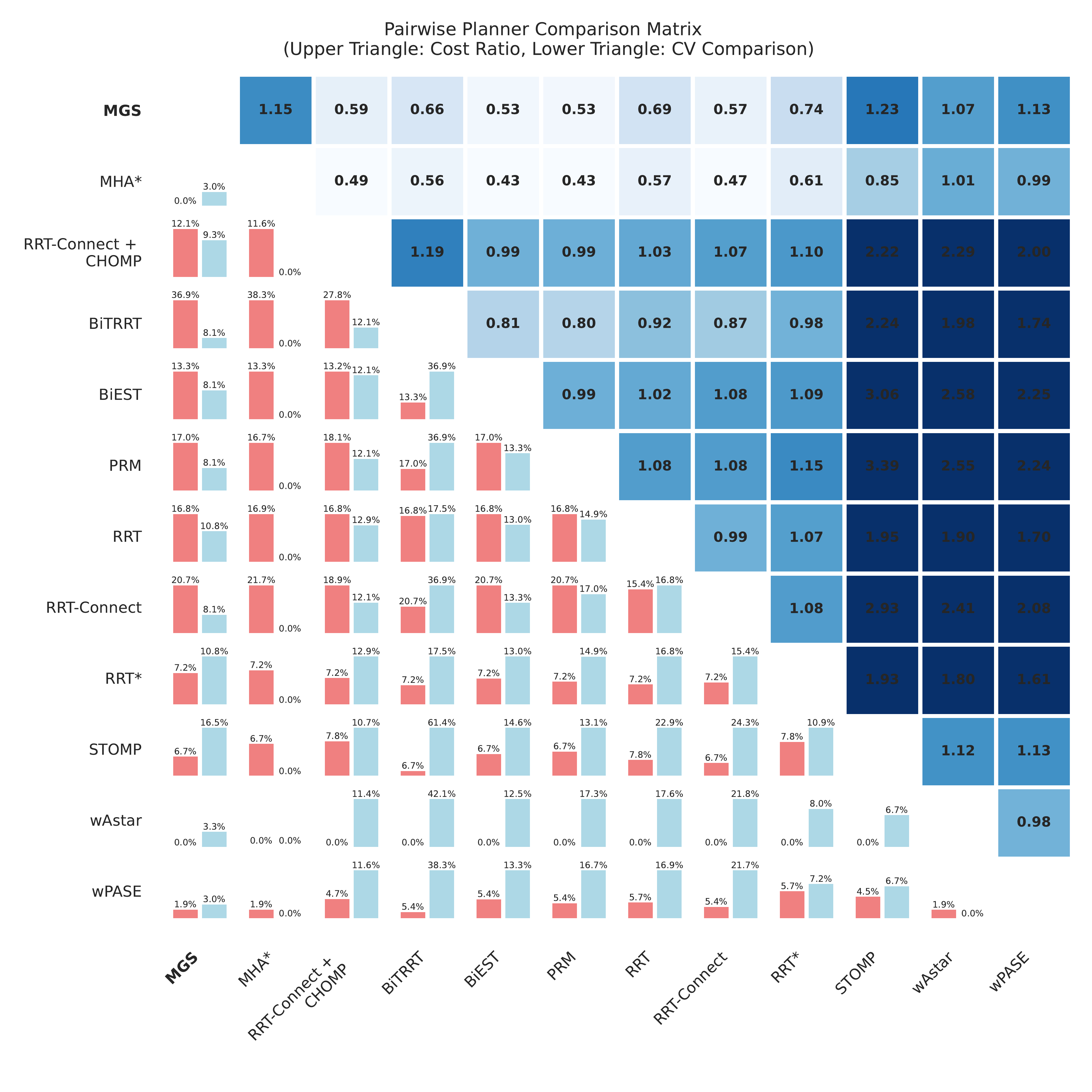}
        \caption{Low-clearance passage.}
    \end{subfigure}
    \hfill
    \begin{subfigure}[b]{0.48\textwidth}
        \centering
        \includegraphics[width=\columnwidth]{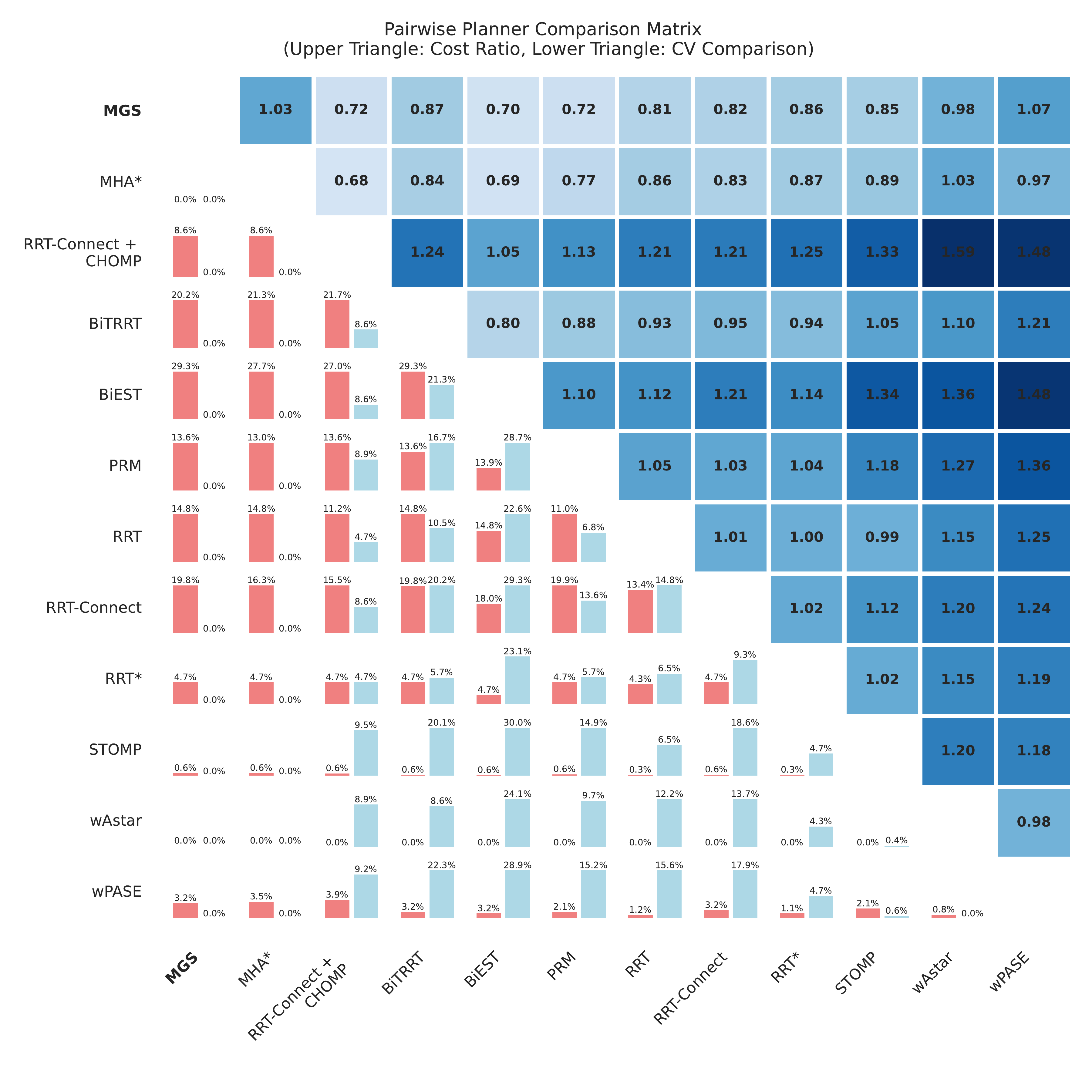}
        \caption{Deep shelf reach.}
    \end{subfigure}
    \\[0.5em]
    \begin{subfigure}[b]{0.48\textwidth}
        \centering
        \includegraphics[width=\columnwidth]{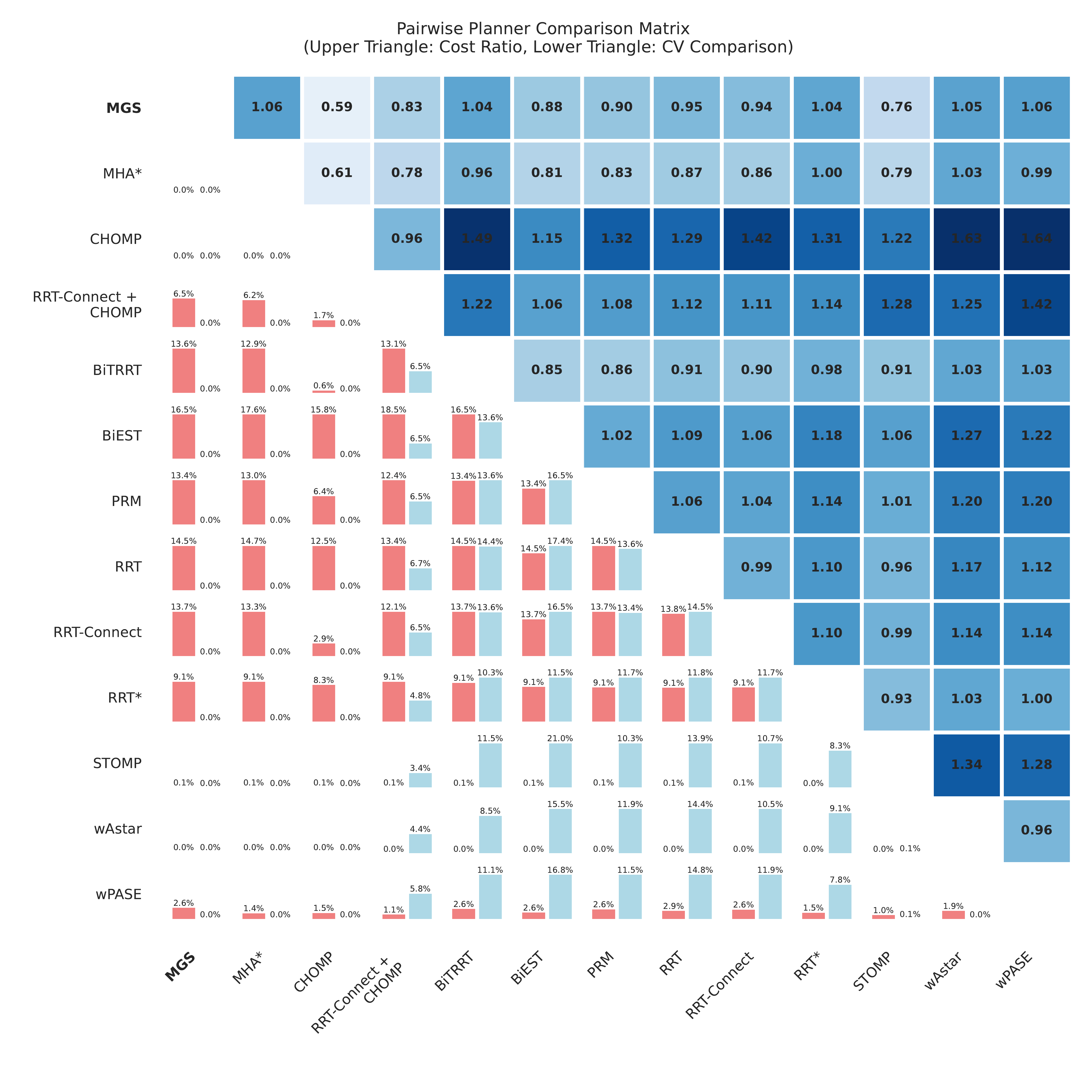}
        \caption{Cluttered table.}
    \end{subfigure}
    \hfill
    \begin{subfigure}[b]{0.48\textwidth}
        \centering
        \includegraphics[width=\columnwidth]{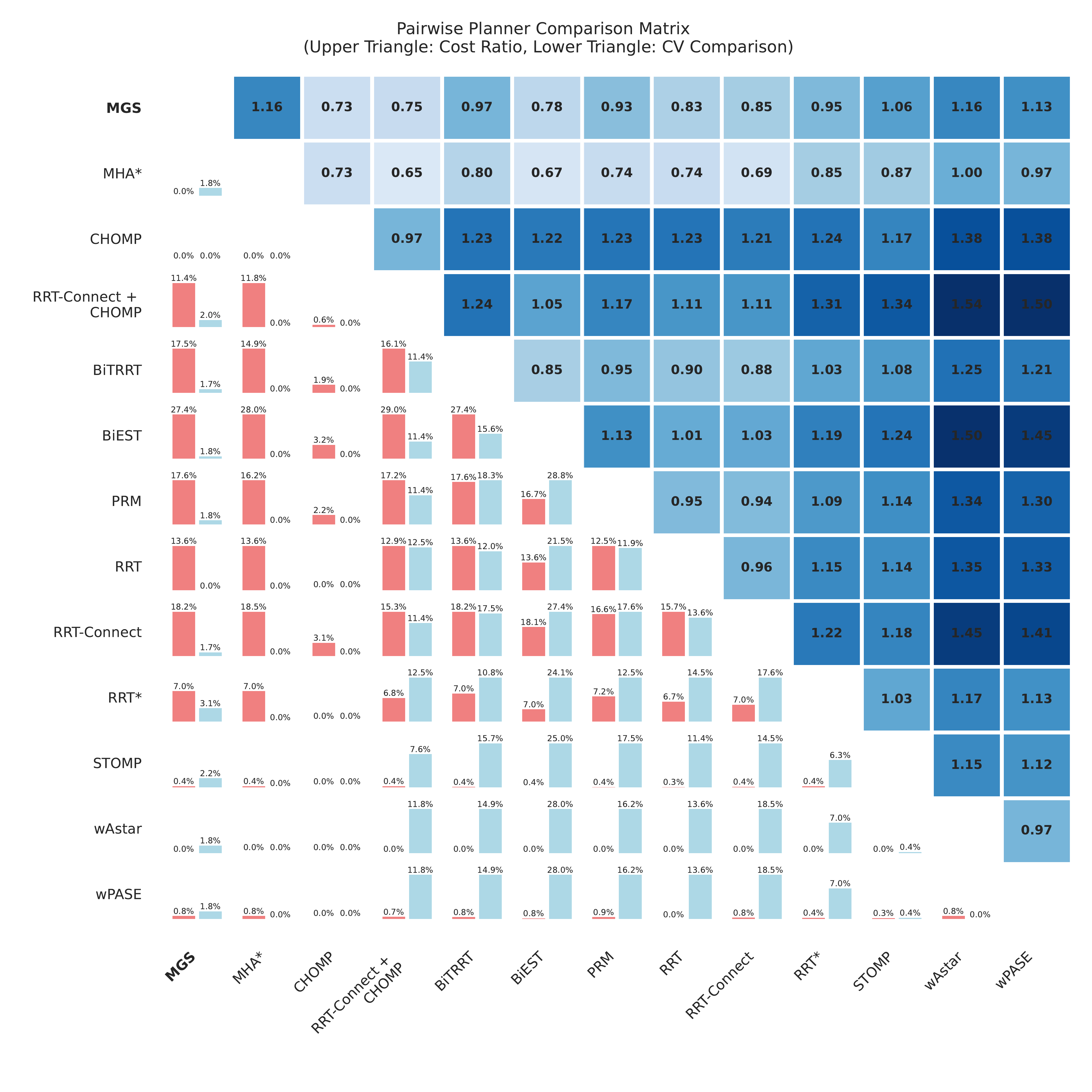}
        \caption{Combined warehouse.}
    \end{subfigure}
    \caption{Pairwise comparison matrices for mobile manipulation scenarios (9-DOF Ridgeback + UR10e). The upper triangle shows the pairwise relative cost ratio (row planner divided by column planner), computed only on queries where both planners succeeded---values below 1.0 indicate the row planner produces lower-cost solutions. The lower triangle shows the coefficient of variation (CV) of these cost ratios, measuring consistency---lower values indicate more predictable pairwise behavior.}
    \label{fig:per_scenario_conf_mobile}
\end{figure*}

\begin{figure*}[htbp]
    \centering
    \begin{subfigure}[b]{0.48\textwidth}
        \centering
        \includegraphics[width=\columnwidth]{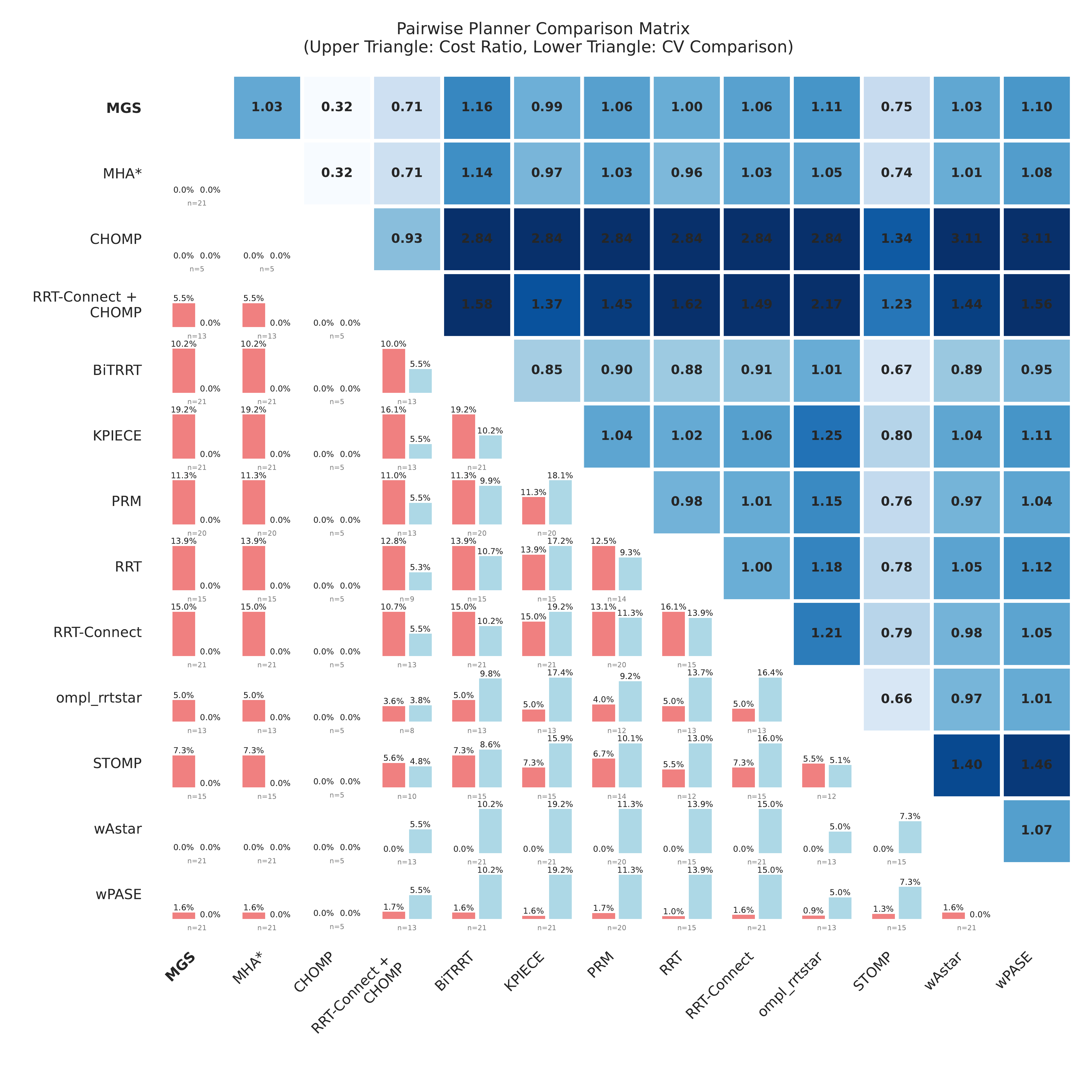}
        \caption{Shelf pick-and-place.}
    \end{subfigure}
    \hfill
    \begin{subfigure}[b]{0.48\textwidth}
        \centering
        \includegraphics[width=\columnwidth]{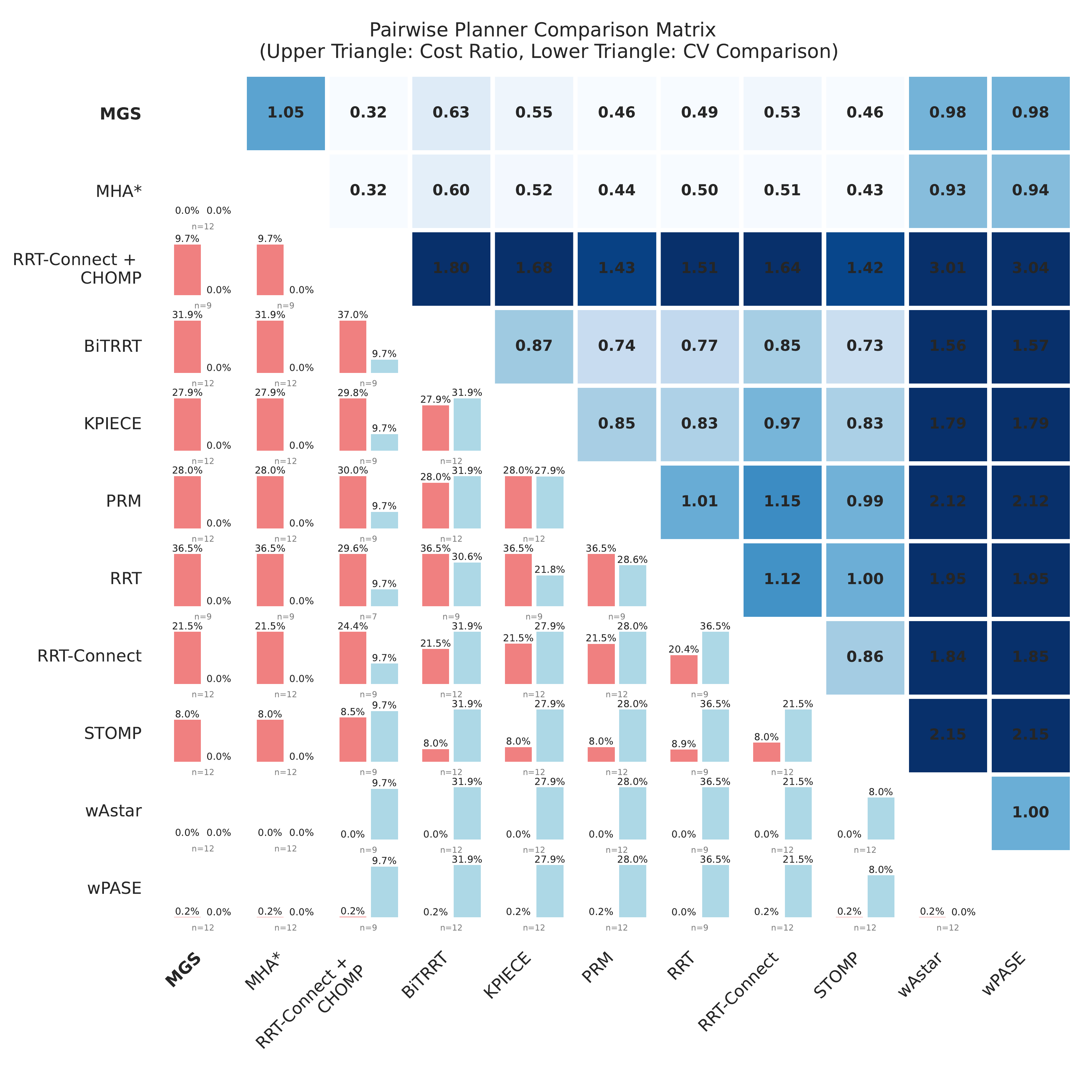}
        \caption{Bin picking.}
    \end{subfigure}
    \\[0.5em]
    \begin{subfigure}[b]{0.48\textwidth}
        \centering
        \includegraphics[width=\columnwidth]{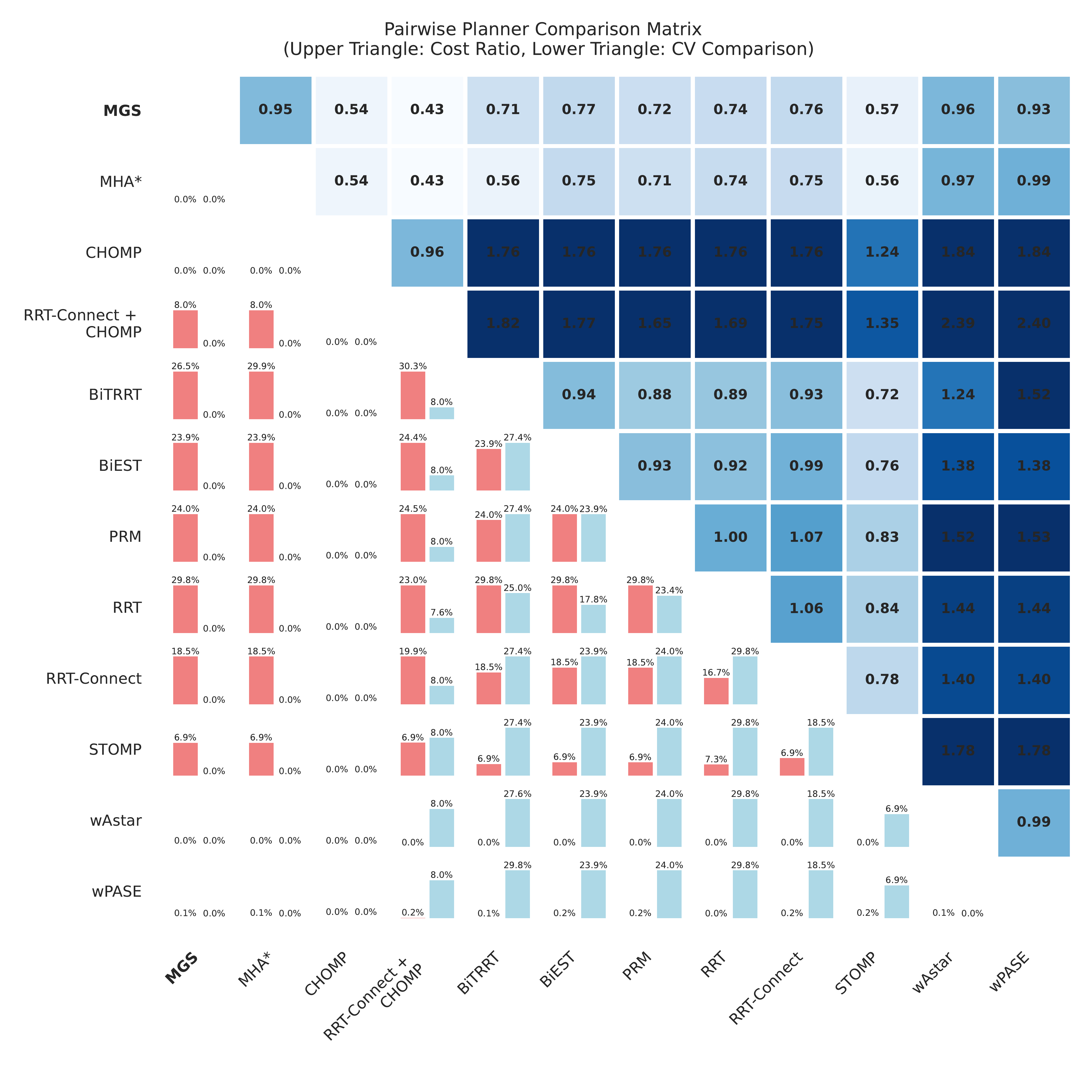}
        \caption{Cage extraction.}
    \end{subfigure}
    \hfill
    \begin{minipage}[b][0.5\textwidth][c]{0.48\textwidth}
        \caption{Pairwise comparison matrices for manipulation scenarios (7-DOF Franka Panda). As in Fig.~\ref{fig:per_scenario_conf_mobile}, the upper triangle shows pairwise relative cost ratios and the lower triangle shows consistency (CV). Results exhibit similar trends: \mgs{} produces competitive solution costs with low variance across queries.}
        \label{fig:per_scenario_conf_manip}
    \end{minipage}
\end{figure*}

\subsection{Computational Overhead vs.\ Number of Sub-graphs}
\label{subsec:appendix_complexity}

We analyze how the computational overhead of \mgs{} scales with the number of sub-graphs.
Fig.~\ref{fig:complexity_vs_subgraphs} shows the ratio of time spent in auxiliary operations---including \textsc{TryToConnect}, \textsc{MergeSubGraphs}, and sub-graph bookkeeping---to the time spent in state expansion.
The overhead exhibits an approximately linear trend with respect to the number of sub-graphs: each additional sub-graph introduces connection attempts and potential merge operations that scale with the number of existing sub-graphs.

Empirically, we found that $m = 10$ sub-graphs provides a good trade-off between performance gains and computational cost.
Beyond this point, the improvements in success rate and solution quality begin to stagnate, while the overhead continues to grow linearly.
This observation guided our choice of the default sub-graph budget used throughout the experiments.

\subsection{Generalization}
\mgs{} is designed as a general framework for planning with undirected graphs.
The multi-directional approach is a general planning paradigm, complemented by a root selection procedure that can be adapted to different domains.
To demonstrate this generalization, we applied \mgs{} to 3D navigation problems (x-y-$\theta$) with a footprint (e.g., planning for the base of a mobile manipulator requires full-body collision checking).
We constructed occupancy grids of the environment, computed attractors in the 2D workspace (x-y) as in Section~\ref{subsec:root_selection}, and mapped them to (x-y-$\theta$) configurations using a simple heuristic: for attractors closer to the start, we set $\theta$ to match the start orientation; for attractors closer to the goal, we set $\theta$ to match the goal orientation.
Collision checking is performed for the entire robot---mobile base and arm.
See Fig.~\ref{fig:2d_occ_grid} for an example environment, its occupancy grid representation, and a path planned by \mgs{}.

\begin{figure}[htbp]
    \centering
    \begin{subfigure}{\columnwidth}
        \centering
        \includegraphics[width=0.95\columnwidth]{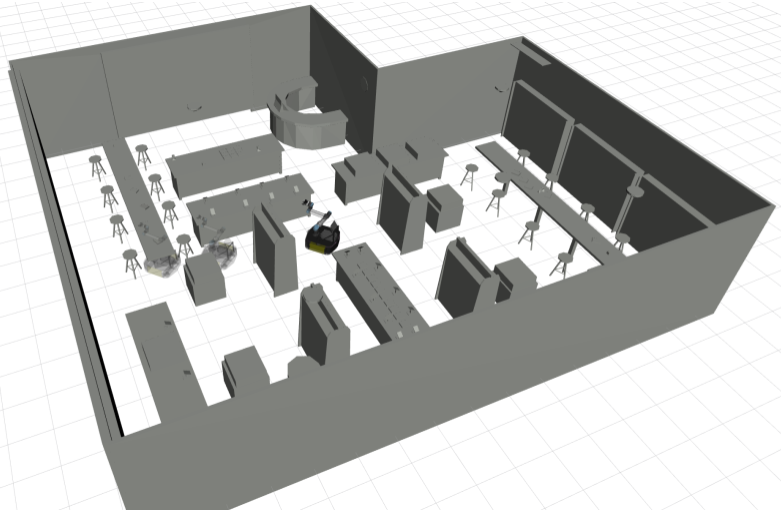}
        \caption{AWS RobotMaker Bookstore World Environment \href{https://github.com/aws-robotics/aws-robomaker-bookstore-world}{https://github.com/aws-robotics/aws-robomaker-bookstore-world}}
    \end{subfigure}
    \hfill
    \begin{subfigure}{\columnwidth}
        \centering
        \includegraphics[width=0.95\columnwidth]{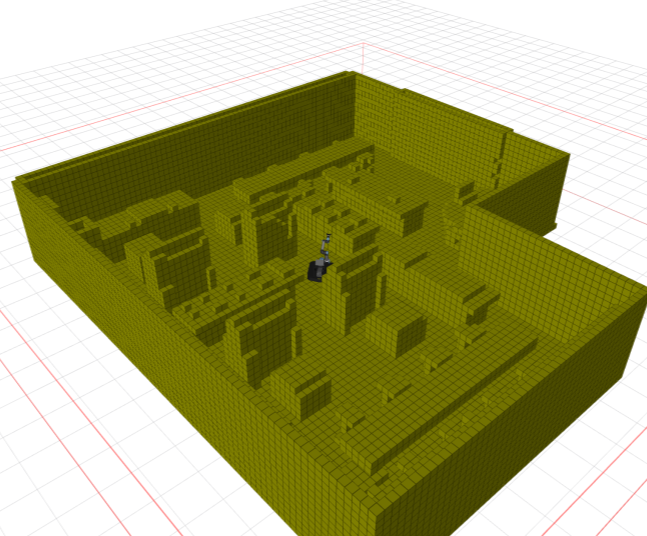}
        \caption{Occupancy grid}
    \end{subfigure}
    \hfill
    \begin{subfigure}{\columnwidth}
        \centering
        \includegraphics[width=0.95\columnwidth]{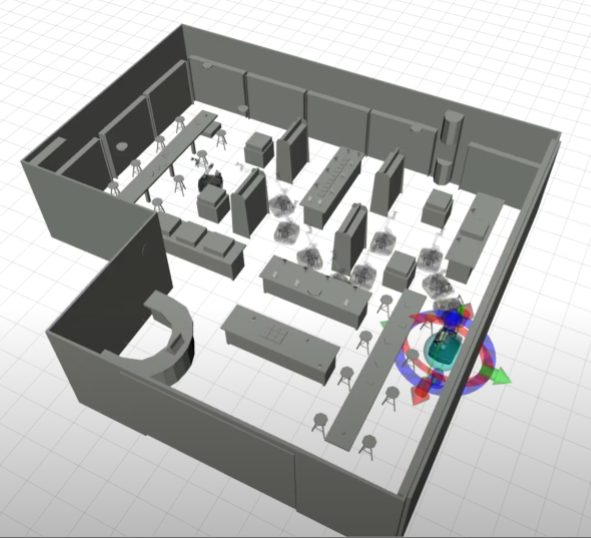}
        \caption{Path planned by \mgs{}}
    \end{subfigure}
    \caption{Example 3D navigation environment (top) and its corresponding occupancy grid representation (middle). The occupancy grid is used for attractor computation in the root selection procedure. The bottom image shows a path planned by \mgs{} in this environment, demonstrating the framework's applicability beyond manipulation tasks.}
    \label{fig:2d_occ_grid}
\end{figure}

\subsection{Impact of Joint Limits}
\label{subsec:joint_limits}

Search-based methods, including \mgs{}, explore states systematically via successor generation from the current state.
As a result, performance is largely unaffected by the range of joint limits: increasing the limits of a particular joint (e.g., the base joints of a mobile manipulator) does not alter the search behavior or solution quality, provided the start and goal configurations remain reachable.

Sampling-based planners, in contrast, are highly sensitive to joint-limit ranges.
Wider joint limits enlarge the sampling domain, diluting the probability of sampling configurations in task-relevant regions.
This leads to degraded planning performances including success rate and planning time, and the solutions are often highly suboptimal.
The effect is particularly pronounced in mobile manipulation, where the base joints can have large ranges relative to the arm joints, skewing the sampling distribution away from the manipulation workspace.

\subsection{Failure Cases}
\label{subsec:failure_cases}

While \mgs{} improves performance in many scenarios, several limitations exist.
The root selection procedure computes attractors via BFS in the workspace, which may not reflect configuration-space connectivity; if an obstacle blocks the arm but not the end-effector path, no attractor will be placed to help navigate around it.
Similarly, if a critical narrow passage in configuration space does not coincide with any attractor, \mgs{} gains no advantage over standard bidirectional search.
Mapping workspace attractors to configuration-space roots relies on differential IK, which may fail in highly constrained scenarios or near singularities, causing some attractors to be discarded.
Additionally, when paths are relatively direct, maintaining multiple sub-graphs introduces overhead without proportional benefit, and a well-guided unidirectional search may outperform \mgs{}.

Finally, because attractors are identified where the BFS wavefront diverges around obstacles, they inherently ``hug'' obstacle boundaries, and paths through these intermediate goals tend to remain close to obstacles.
In applications where maximizing clearance is desirable---such as safety-critical tasks or environments with uncertain obstacle geometry---this behavior may be undesirable, and alternative attractor selection strategies that balance efficiency with clearance could address this limitation.

\end{document}